\documentclass[a4paper,11pt]{article}
\usepackage{tikz}
\usepackage{amsmath}
\usepackage{amsthm}
\usepackage{amssymb}
\usepackage{amsfonts}
\usepackage{diagbox}
\usepackage[font={small}, center]{caption}
\usepackage[all]{xy}

\usepackage[round]{natbib}
\RequirePackage[colorlinks,citecolor=blue,urlcolor=blue]{hyperref}
\usepackage{todonotes}

\usepackage{algorithm}
\usepackage{algorithmic}
\usepackage{microtype}
\usepackage{graphicx}
\usepackage{subfigure}
\usepackage{booktabs}

\newcommand{\eins}{\boldsymbol{1}}

\newcommand{\argmin}{\operatornamewithlimits{arg \, min}}
\DeclareSymbolFont{wideparensymbol}{OMX}{yhex}{m}{n}
\DeclareMathAccent{\wideparen}{\mathord}{wideparensymbol}{"F3}
\usepackage{multirow}
\usepackage{threeparttable}
\usepackage{wrapfig}

\usepackage[capitalize,noabbrev]{cleveref}

\theoremstyle{plain}
\newtheorem{theorem}{Theorem}[section]
\newtheorem{proposition}[theorem]{Proposition}
\newtheorem{lemma}[theorem]{Lemma}

\theoremstyle{definition}
\newtheorem{definition}[theorem]{Definition}
\newtheorem{assumption}[theorem]{Assumption}
\theoremstyle{remark}

\theoremstyle{example}
\newtheorem{example}[theorem]{Example}

\begin{document}

\def\spacingset#1{\renewcommand{\baselinestretch}%
{#1}\small\normalsize} \spacingset{1}


  \title{\bf On the Robustness of Kernel Ridge Regression Using the Cauchy Loss Function}
  \author{
  Hongwei Wen\\
     \hspace{.2cm}\\
University of Twente\\   
              \texttt{h.wen@utwente.nl}
  \and
    Annika Betken\\
\hspace{.2cm}\\
              University of Twente\\   
              \texttt{a.betken@utwente.nl}\\
              \and
      Wouter Koolen\\
     \hspace{.2cm}\\
    University of Twente \\
              \texttt{w.m.koolen@utwente.nl}
}              
  \maketitle

\bigskip
\begin{abstract}
Robust regression aims to develop methods for estimating an unknown regression function in the presence of outliers, heavy-tailed distributions, or contaminated data, which can severely impact performance. Most existing theoretical results in robust regression assume that the noise has a finite absolute mean, an assumption violated by certain distributions, such as Cauchy and some Pareto noise. In this paper, we introduce a generalized Cauchy noise framework that accommodates all noise distributions with finite moments of any order, even when the absolute mean is infinite. Within this framework, we study the \textit{kernel Cauchy ridge regressor} (\textit{KCRR}), which minimizes a regularized empirical Cauchy risk to achieve robustness. To derive the $L_2$-risk bound for KCRR, we establish a connection between the excess Cauchy risk and $L_2$-risk for sufficiently large scale parameters of the Cauchy loss, which reveals that these two risks are equivalent. Furthermore, under the assumption that the regression function satisfies H\"older smoothness, we derive excess Cauchy risk bounds for KCRR, showing improved performance as the scale parameter decreases. By considering the twofold effect of the scale parameter on the excess Cauchy risk and its equivalence with the $L_2$-risk, we establish the almost minimax-optimal convergence rate for KCRR in terms of $L_2$-risk, highlighting the robustness of the Cauchy loss in handling various types of noise. Finally, we validate the effectiveness of KCRR through experiments on both synthetic and real-world datasets under diverse noise corruption scenarios.

\end{abstract}



\section{Introduction} \label{sec::Intro}

Robust regression seeks to accurately estimate the true regression function in scenarios where outliers, heavy-tailed distributions, or contaminated data points can significantly distort the results of standard regression methods. Mathematically, it aims to provide accurate estimates even when the noise or response distribution lacks a finite exponential expectation, meaning the distribution may have heavy tails but still possess a finite $p$-th moment. Due to its resilience against such deviations from typical assumptions, robust regression has been extensively applied in fields such as finance \citep{pervez2022robust}, economics \citep{khan2021applications}, engineering \citep{agrusa2022robust}, and environmental science \citep{pirtea2021interplay}.

Many studies on robust regression assume that the absolute mean of the noise is finite, typically represented as noise with a finite $p$-th moment for $p \in [1, \infty)$. For instance, under the assumption of a finite $p$-th moment with $p \in (2,4]$, \cite{feng2015learning} investigate a correntropy-induced regression loss \citep{santamaria2006generalized} and derive error bounds for the $L_2$-distance between the learned regressor and the true regression function. When the noise has finite variance, i.e., $p = 2$, \cite{catoni2012challenging} proposes a log-truncated loss function for mean and variance estimation, demonstrating that the deviation of the estimator is comparable to that of the empirical mean estimator for Gaussian-tailed data, while \cite{xu2020learning} establish excess risk bounds for the empirical log-truncated risk minimizer under the finite variance assumption.
Further research has extended the use of Catoni’s log-truncated loss to cases where the noise has a finite $p$-th moment with $p \in (1,2)$. For example, this loss function has been applied to \textit{least absolute deviation} (\textit{LAD}) regression \citep{chen2021generalized}, mean estimation \citep{lam2021robust}, and multi-armed bandits \citep{lee2020optimal}. Moreover, \cite{xu2023non} propose a generalized form of the log-truncated loss, which is used for solving quantile regression and \textit{generalized linear models} (\textit{GLM}). They establish excess risk bounds for the empirical log-truncated risk minimizer with respect to both the pinball loss and the GLM loss.

In many applications, the noise distribution does not have a finite absolute mean, i.e., the noise has only a finite $p$-th moment for $p \in (0, 1)$, as seen with Cauchy noise and certain Pareto noise distributions. Cauchy noise frequently arises in finance and markets \citep{ljajko2023cauchy}, radar images \citep{karakucs2022cauchy}, low-frequency atmospheric noise \citep{fu2010research}, and underwater acoustic signals \citep{idan2010cauchy, shi2020cauchy}. It is also common in industrial wireless communication systems \citep{laus2018nonlocal, zhang2021measurements}, particularly in massive multiple-input multiple-output systems \citep{gulgun2023massive}. Meanwhile, Pareto noise is often used to model wealth distribution in societies. In such cases, the typical assumption of a finite $p$-th moment with $p \in (1, \infty)$ is violated, rendering existing theoretical guarantees inapplicable. Despite this, the Cauchy loss function \citep{black1996robust} has empirically shown promising results in regression problems under extreme noise conditions \citep{mlotshwa2022cauchy}. However, to the best of our knowledge, no existing work has established error bounds for the empirical Cauchy risk minimizer in non-parametric regression problems when the noise has only a finite $p$-th moment with $p \in (0, 1)$.

In this paper, we propose a generalized Cauchy noise assumption, which assumes that the logarithmic moment of the noise is finite. Since the logarithmic function grows more slowly than any polynomial function, this assumption can be satisfied by noise with a finite $p$-th moment for any $p \in (0, \infty)$. Under this assumption, we find that, unlike the expected absolute loss or Huber loss \citep{huber1992robust}, the expected Cauchy loss \citep{black1996robust} remains finite, and its Bayes function coincides with the true regression function. This confirms that the regression function can be effectively learned by minimizing the Cauchy loss in the presence of various noise types. Based on this insight, we develop a kernel-based regression method called the \textit{kernel Cauchy ridge regressor} (\textit{KCRR}), which minimizes the regularized empirical Cauchy risk within a \textit{reproducing kernel Hilbert space} (\textit{RKHS}).

Since $L_2$-risk is a widely used criterion in robust regression for measuring the $L_2$-distance between the regressor and the true regression function, our goal is to establish the convergence rate of the KCRR with respect to $L_2$-risk. To achieve this, we express the $L_2$-risk of KCRR as the product of two terms: the supremum of the ratio between the $L_2$-risk and the excess Cauchy risk across all regressors, and the excess Cauchy risk of KCRR itself.
To bound the supremum of this risk ratio, we first show that the $L_2$-risk can be bounded by the sum of the excess Cauchy risk and an $L_4$-risk term, which decreases as $\sigma^2$ grows. Upon further analysis, we demonstrate that for large values of $\sigma$, the $L_4$-risk term can be bounded by an $L_2$-risk term. This leads to a calibration inequality, indicating that the $L_2$-risk becomes equivalent to the excess Cauchy risk when $\sigma$ is sufficiently large.
Next, to derive the upper bound of the excess Cauchy risk for KCRR, we decompose it into two components: the sample error and the approximation error. Benefiting from the refined calibration inequality, we show that the Cauchy loss satisfies the variance bound with an optimal exponent, which is crucial for deriving a promising oracle inequality for KCRR. Additionally, since the excess Cauchy risk is smaller than the $L_2$-risk for any $\sigma$, the approximation error with respect to the Cauchy loss can be bounded by the approximation error in terms of $L_2$-risk.
By combining the upper bound of the risk ratio with the excess Cauchy risk of KCRR, we are able to establish an almost minimax-optimal convergence rate for KCRR in terms of $L_2$-risk, given appropriate parameter choices.

The contributions of this paper can be summarized as follows:

\textit{(i)} We introduce a generalized Cauchy noise assumption, which encompasses a broad range of noise distributions, including those without a finite absolute mean, such as Cauchy noise and certain Pareto noise. To tackle the robust regression problem under this challenging noise setting, we propose the KCRR and derive excess Cauchy risk bounds for KCRR. Under the assumption that the regression function is H\"{o}lder smooth, these bounds decrease as the scale parameter $\sigma$ decreases.

\textit{(ii)} Under the generalized Cauchy noise assumption, we establish a link between the excess Cauchy risk and the $L_2$-risk. Specifically, for sufficiently large values of the Cauchy loss scale parameter $\sigma$, minimizing the two risks becomes equivalent. This shows that minimizing the Cauchy loss on noisy samples is equivalent to minimizing the $L_2$-distance between the regressor and the true regression function, making the learning process robust to noise. This equivalence highlights why using the Cauchy loss leads to a resilient regressor in the presence of extreme noise.

\textit{(iii)} Building on these results, we demonstrate that the scale parameter $\sigma$ plays a crucial role in both the excess Cauchy risk bound from \textit{(i)} and the equivalence between the Cauchy risk and $L_2$-risk established in \textit{(ii)}. By selecting an appropriate $\sigma$, we achieve an almost minimax-optimal convergence rate for KCRR in terms of $L_2$-risk, underscoring the robustness of the Cauchy loss in mitigating the effects of extreme noise.

\textit{(iv)} We conduct experiments on synthetic and real-world regression datasets, demonstrating that KCRR outperforms existing methods under various types of noise corruption.

The remainder of this paper is organized as follows: In Section \ref{sec::robustreg}, we define the robust regression problem, introduce the generalized Cauchy noise assumption, explore the properties of the Cauchy loss function, and propose the KCRR method. In Section \ref{sec::mainresults}, we establish the connection between the excess Cauchy risk and $L_2$-risk, and under mild assumptions, derive the almost minimax-optimal convergence rate of KCRR in terms of $L_2$-risk, comparing it with existing rates. Section \ref{sec::ErrorAnalysis} provides a detailed error analysis for the $L_2$-risk of KCRR. Experimental results and empirical comparisons with other loss functions are presented in Section \ref{sec::experiments}. Finally, the proofs for Sections \ref{sec::robustreg}--\ref{sec::ErrorAnalysis} are provided in Section \ref{sec::proofs}, and the paper concludes with Section \ref{sec::conclusion}.

\section{Robust Regression: Kernel Cauchy Ridge Regression} \label{sec::robustreg}

In this section, we first define the robust regression problem in Section \ref{sec::robustregIntro}, outlining the challenges posed by outliers and heavy-tailed noise. In Section \ref{sec::GeneralizedCauchy}, we introduce the generalized Cauchy noise assumption, which provides a framework for modeling such noise. We then delve into the characteristics of the Cauchy loss function in Section \ref{sec::CauchyLoss}, highlighting its robustness against extreme noise and outliers. Finally, in Section \ref{sec::KCRRmethod}, we propose the \textit{kernel Cauchy ridge regression} (\textit{KCRR}) method, which combines the Cauchy loss with kernel-based techniques to provide a powerful approach for robust regression.

\subsection{Robust Regression} \label{sec::robustregIntro}

In this paper, we consider a regression framework where $\mathcal{X} \subset \mathbb{R}^d$ is a compact, non-empty set and $\mathcal{Y} := \mathbb{R}$. The goal of the regression problem is to predict the value of an unobserved response variable $Y \in \mathcal{Y}$ based on the observed value of an explanatory variable $X \in \mathcal{X}$. We formulate the regression model as follows:
\begin{align}\label{equ::regressionmodel}
	Y = f^*(X) + \epsilon,
\end{align}
where $f^*: \mathcal{X} \to \mathcal{Y}$ represents the true regression function with $\|f^*\|_{\infty} < \infty$ and $\epsilon$ denotes the noise variable, which is assumed to be independent of $X$.
In this context, outliers and heavy-tailed noise, represented by the noise variable $\epsilon$, pose significant challenges for learning $f^*$. Outliers can disproportionately influence traditional loss functions such as squared loss, absolute loss, and Huber loss, leading to biased estimates and diminished generalization performance. Moreover, heavy-tailed noise allows for large deviations in $\epsilon$, exacerbating issues related to model stability and performance.
Robust regression methods aim to address these challenges by utilizing loss functions designed to minimize the impact of extreme values. By reducing the influence of outliers and heavy-tailed noise, robust regression ensures more reliable predictions, even in the presence of such disruptions.

Let the joint probability distribution of $(X,Y)$ in \eqref{equ::regressionmodel} on the space $\mathcal{X} \times \mathcal{Y}$ be denoted by $P$. Consider a set of $n$ independent and identically distributed samples $D:=\left\{(X_1, Y_1), \ldots, (X_n, Y_n)\right\}$ drawn from $P$. The objective is to find a regressor $f:\mathcal{X} \to \mathcal{Y}$ that estimates the true regression function $f^*$ based on the observations $D$ generated by model \eqref{equ::regressionmodel}. 
For any measurable function $f: \mathcal{X} \to \mathcal{Y}$, the population risk and empirical risk are defined respectively as follows:
\begin{align*}
	\mathcal{R}_{L,P}(f) := \int_{\mathcal{X} \times \mathcal{Y}} L(y, f(x)) \, dP(x,y), 
	\qquad
	\mathcal{R}_{L,\mathrm{D}}(f) := \frac{1}{n}\sum_{i=1}^n L(Y_i, f(X_i)),
\end{align*}
where $\mathrm{D} := n^{-1}\sum_{i=1}^{n} \delta(X_i,Y_i)$ represents the empirical measure based on the data $D$, and $\delta(X_i,Y_i)$ is the Dirac measure at $(X_i,Y_i)$. 
The Bayes risk, which represents the minimal achievable risk with respect to $P$ and $L$, is given by
\begin{align*}
	\mathcal{R}_{L, P}^{*}:= \inf \{\mathcal{R}_{L, P}(f) \mid f: \mathcal{X} \to \mathcal{Y} \text{ measurable} \}.
\end{align*}
Moreover, a measurable function $f:\mathcal{X} \to \mathcal{Y}$ that satisfies $\mathcal{R}_{L,P}(f) = \mathcal{R}_{L,P}^{*}$ is called a Bayes function for the loss function $L$ and distribution $P$.

Before proceeding, we introduce some notation that will be used throughout this paper.
Let $(\Omega, \mathcal{A}, \nu)$ be a probability space. For $1 \leq p \leq \infty$, we denote by $L_p(\nu)$ the space of measurable functions $g: \Omega \to \mathbb{R}$ with a finite $L_p$-norm. Specifically, for $1 \leq p < \infty$, the $L_p$-norm is defined as
$\|g\|_{L_p(\nu)} := ( \int_{\Omega} |g(x)|^p \, d\nu(x) )^{1/p}$,
and for $p = \infty$, we define
$\|g\|_{\infty} := \inf\{ M > 0 : |g(x)| \leq M \text{ for almost every } x \in \Omega\}$.
The space $L_p(\nu)$, equipped with the norm $\|g\|_{L_p(\nu)}$, forms a Banach space. For any Banach space $E$, we denote its closed unit ball by $B_E$. 
Additionally, we use the notation $a_n \lesssim b_n$ (or $a_n \gtrsim b_n$) to indicate that there exists a constant $c > 0$ such that $a_n \leq c b_n$ (or $a_n \geq c b_n$) for all $n \in \mathbb{N}$. We write $a_n \asymp b_n$ if there exists a constant $c \in (0, 1]$ such that $c b_n \leq a_n \leq c^{-1} b_n$ for all $n \in \mathbb{N}$. 
For a natural number $n$, we denote $[n] := \{1, 2, \ldots, n\}$. Lastly, for $a, b \in \mathbb{R}$, we define $a \vee b := \max\{a, b\}$, representing the larger of the two values.

\subsection{Generalized Cauchy Noise} \label{sec::GeneralizedCauchy}

In the existing literature on robust regression, it is commonly assumed that a given $p$-th moment of the noise variable $\epsilon$ in \eqref{equ::regressionmodel} is finite, meaning that $\mathbb{E} |\epsilon|^p < \infty$, for some $p \in (1, \infty)$. For example, \cite{feng2015learning, Brownlees2015empirical} assume $p \in (2, 4]$, while \cite{xu2020learning, zhang2018ell_1} consider the case where $p = 2$, and \cite{xu2023non,shen2021deep,shen2021robust} extend the analysis to $p \in (1, 2)$. These results apply to noise distributions with fast-decaying tails, such as Gaussian noise, and moderately heavy-tailed distributions like the Student-$t$ noise with degrees of freedom strictly greater than $2$. However, as discussed in Section \ref{sec::Intro}, many real-world scenarios involve data contaminated by extremely heavy-tailed noise, such as Cauchy noise or certain types of Pareto noise, which do not have finite variance or even a finite absolute mean. These cases, illustrated in the following examples, fall outside the scope of traditional assumptions.

\begin{example}[Cauchy Noise]\label{ex::cauchy}
	The probability density function of the Cauchy distribution is given by
	\begin{align*}
		p(\epsilon) = \frac{1}{\pi s(1+ \epsilon^2 / s^2)},
		\qquad 
		\epsilon \in \mathbb{R},
	\end{align*}
	where $s \in (0, \infty)$ is the scale parameter. Notably, for any $p \in [1, \infty)$, the absolute moment $\mathbb{E} |\epsilon|^p = \infty$, meaning that neither the mean nor variance of Cauchy noise are finite. However, for $p \in (0, 1)$, we have $\mathbb{E} |\epsilon|^p = s^p \sec(\pi p / 2)$, indicating that the Cauchy noise satisfies the finite $p$-th moment assumption only when $p \in (0, 1)$.
\end{example}

\begin{example}[Pareto Noise]\label{ex::pareto}
	The probability density function of the (generalized) Pareto noise distribution is given by
	\begin{align*}
		p(\epsilon) = \frac{1}{2s} \left(1 + \frac{\zeta |\epsilon|}{s}\right)^{-(1 + 1/\zeta)},
		\qquad \epsilon \in \mathbb{R},
	\end{align*}
	where $s \in (0, \infty)$ is the scale parameter and $\zeta \in (0, \infty)$ is the shape parameter. For Pareto noise with shape parameter $\zeta$, the $p$-th moment exists only if $\zeta \in (0, 1/p)$. In particular, when $\zeta \in [1/2, \infty)$, the variance becomes infinite, and when $\zeta \in [1, \infty)$, the absolute mean is also infinite.
\end{example}

It is important to note that both Cauchy noise and Pareto noise with $\zeta \in [1, \infty)$ do not satisfy the finite $p$-th moment assumption for $p \in [1, \infty)$; they only meet this condition for some $p \in (0, 1)$. To account for all noise distributions across the range $p \in (0, \infty)$, we propose the following assumption.

\begin{assumption}[Generalized Cauchy Noise]\label{ass::logmoment}
	We assume that the noise variable $\epsilon$ in model \eqref{equ::regressionmodel} has a finite logarithmic moment, specifically, $\mathbb{E} \bigl( \log(1 + \epsilon^2) \bigr) < \infty$.
\end{assumption}

Clearly, because the logarithmic function grows more slowly than any polynomial function, Assumption \ref{ass::logmoment} encompasses noise distributions with finite $p$-th moments for all $p \in (0, \infty)$, as shown in the following lemma.

\begin{lemma}\label{lem::stronger}
	For any $p>0$, if a noise distribution has a finite $p$-th moment, i.e., $\mathbb{E} |\epsilon|^p < \infty$, then it satisfies Assumption \ref{ass::logmoment}, i.e., $\mathbb{E} \bigl( \log(1 + \epsilon^2) \bigr) < \infty$.
\end{lemma}

For simplicity, we introduce the following assumptions: the noise is symmetrically distributed around zero and exhibits monotonically decreasing tails.

\begin{assumption}[Symmetric Noise]\label{ass::symmetry}
	Assume that the noise variable $\epsilon$ is symmetrically distributed, meaning its probability density function $p_{\epsilon}$ satisfies $p_{\epsilon}(t) = p_{\epsilon}( -t)$ for any $t \in \mathbb{R}$.
\end{assumption}

Assumption \ref{ass::symmetry} is clearly met by many commonly used noise distributions, such as Gaussian, Student-$t$, Laplace, and Cauchy noise. This symmetry assumption is frequently employed in robust regression studies (see, e.g., \cite{d2021consistent,Convergence2014Tsakonas}).

\begin{assumption}[Monotonically Decreasing Tails]\label{ass::decreasingtails}  
	The noise variable $\epsilon$ is assumed to have monotonically decreasing tails. Specifically, for any $t, t' \in \mathbb{R}$ such that $0 \leq |t| < |t'|$, it holds that $p_{\epsilon}(t) > p_{\epsilon}(t')$.  
\end{assumption}

It is also evident that many commonly used and widely studied noise distributions satisfy Assumption \ref{ass::decreasingtails}. Examples include Gaussian noise, Student-$t$ noise, Cauchy noise, and Pareto noise.

\subsection{The Cauchy Loss Function} \label{sec::CauchyLoss}

As demonstrated in Lemma \ref{lem::stronger}, certain noise distributions satisfy Assumption \ref{ass::logmoment} but do not possess finite variance or a finite absolute mean, such as Examples \ref{ex::cauchy} and \ref{ex::pareto}. For such noise, commonly used robust regression loss functions, like the absolute loss and Huber loss, become ineffective since the associated risks are infinite.

In this paper, we explore the Cauchy loss function \citep{black1996robust}, which offers greater robustness against outliers compared to traditional loss functions. The Cauchy loss is defined as
\begin{align}\label{equ::cauchyloss}
	L(y, f(x)) = \sigma^2 \cdot \log \left( 1 + \frac{(y - f(x))^2}{\sigma^2} \right),
\end{align}
where $\sigma$ is a parameter that controls the spread of the loss function. For small residuals, $y - f(x)$, the Cauchy loss behaves similarly to the square loss, but for large residuals, it grows logarithmically, reducing the influence of extreme noise. As shown in Figure \ref{fig::compare_loss}, smaller values of $\sigma$ yield smaller Cauchy losses for a given residual, while larger values cause the Cauchy loss to gradually approach the square loss.

The following lemma demonstrates that, under the generalized Cauchy noise Assumption~\ref{ass::logmoment}, the Cauchy risk $\mathcal{R}_{L,P}(f)$ is always finite for any bounded regressor $f$.
\begin{lemma}[Finite Risk]\label{lem::finiterisk}
	Let Assumption \ref{ass::logmoment} hold, and let $L$ represent the Cauchy loss as defined in \eqref{equ::cauchyloss}. Under these conditions, for any bounded regressor $f: \mathcal{X} \to \mathcal{Y}$, its corresponding Cauchy risk $\mathcal{R}_{L,P}(f) = \mathbb{E}_P L(Y, f(X))$ is guaranteed to be finite, i.e.,
	$\mathcal{R}_{L,P}(f) < \infty$.
\end{lemma}

Unlike the infinite \textit{mean squared error} (\textit{MSE}) or \textit{mean absolute error} (\textit{MAE}) in the presence of extreme noise or outliers, the finiteness of the Cauchy risk highlights the robustness of the Cauchy loss function. This demonstrates that even under Cauchy noise and Pareto noise in Examples \ref{ex::cauchy} and \ref{ex::pareto}, the Cauchy loss admits a well-defined empirical risk minimization framework.

The key reason for finite Cauchy risk lies in how the Cauchy loss penalizes large prediction errors $|y - f(x)|$ logarithmically, in contrast to the quadratic penalty of MSE or the linear penalty of MAE. As a result, extreme responses exert less influence on the Cauchy risk, making models trained under this loss function more resistant to the impact of outliers.

The following lemma establishes that the regression function $f^*$ is the exact minimizer of the Cauchy risk.

\begin{lemma}[Optimality]\label{lem::optimality}
	Under Assumptions \ref{ass::logmoment},  \ref{ass::symmetry} and \ref{ass::decreasingtails}, with $L$ defined as the Cauchy loss in \eqref{equ::cauchyloss}, the true regression function $f^*$ attains the Bayes risk, i.e., 
	$\mathcal{R}_{L, P}(f^*) 
	= \mathcal{R}_{L, P}^*$.
\end{lemma}

By combining Lemmas \ref{lem::finiterisk} and \ref{lem::optimality}, we conclude that under the assumption of symmetric noise with a finite logarithm moment, the true regression function $f^*$ consistently yields a smaller Cauchy risk than any other regressor $f: \mathcal{X} \to \mathbb{R}$. Specifically, $\mathcal{R}_{L, P}(f^*) \leq \mathcal{R}_{L, P}(f) < \infty$ for any bounded $f$. This provides a theoretical guarantee that minimizing the finite Cauchy risk leads to the recovery of the true regression function $f^*$.

\vspace{2mm}
\noindent
\textbf{Comparison with the Correntropy Loss}.  
The correntropy loss \citep{santamaria2006generalized} is defined as 
\begin{align}\label{equ::correntropyloss}
	L_{\mathrm{corr}}(y, f(x)) := \sigma^2 \cdot \left( 1 - \exp \left( - \frac{(y - f(x))^2}{\sigma^2} \right) \right),
\end{align}
where $\sigma \in (0, \infty)$ is the scale parameter. As illustrated in Figure \ref{fig::compare_loss}, for small residuals $y - f(x)$, the correntropy loss behaves similarly to the square loss but asymptotically approaches $\sigma$ as the residual increases. Its bounded nature ensures robustness against large errors, as it does not grow indefinitely with increasing residuals. However, its performance heavily depends on the choice of the bandwidth parameter $\sigma$. If $\sigma$ is too small, the loss function becomes overly sensitive to minor noise, whereas a large $\sigma$ may cause the model to ignore significant outliers. In contrast, the Cauchy loss in \eqref{equ::cauchyloss} strikes a more consistent balance between robustness to outliers and sensitivity to small errors.

\begin{figure}
	\centering
	\includegraphics[width=0.86\linewidth]{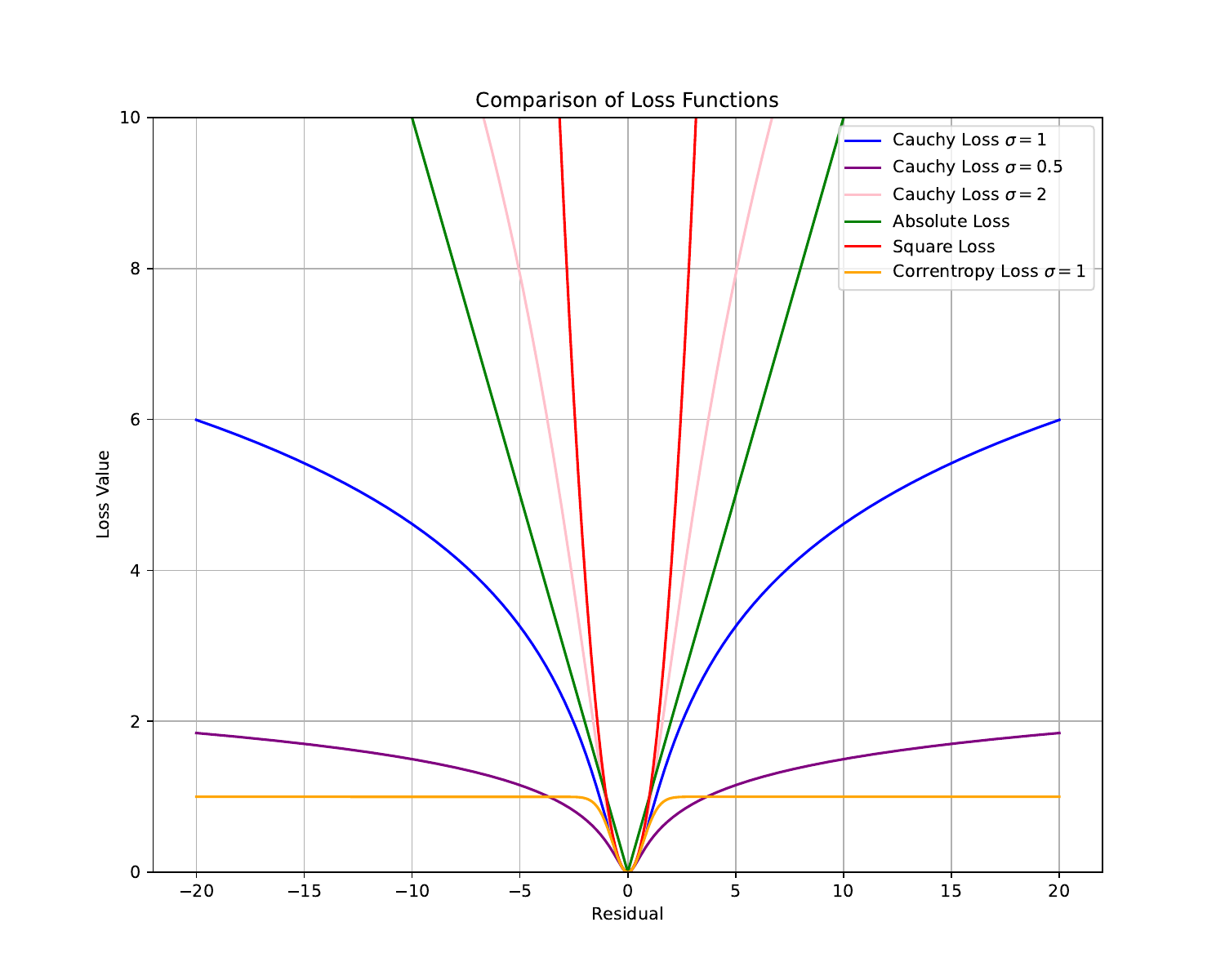}
	\captionsetup{justification=centering}
	\vspace{-6mm}
	\caption{Comparison between different robust loss functions.}
	\label{fig::compare_loss}
\end{figure}

\vspace{2mm}
\noindent
\textbf{Comparison with the Log-Truncated Loss}.  
Log-truncated loss functions, introduced by \cite{catoni2012challenging} for robust learning, take the form 
\begin{align*}
	\ell_{\psi_{\lambda}, s, L}(y, f(x)) := s \cdot \psi_{\lambda} \left( \frac{L(y, f(x))}{s} \right),
\end{align*}
where $\psi_{\lambda}$ is a non-increasing function satisfying 
\begin{align}\label{eq::PsiLambda}
	\log(1 - x + \lambda(|x|)) \leq \psi_{\lambda}(x) \leq \log(1 + x + \lambda(|x|)), \qquad x\geq 0.
\end{align}
In particular, \cite{catoni2012challenging} considers $\lambda(|x|) := |x|^{\beta}/\beta$ with $\beta = 2$, while \cite{xu2023non} explore $\beta \in (1, 2)$. 
When we take $L$ as the absolute loss, $L_{\mathrm{abs}}(y, f(x)) := |y - f(x)|$, set $\beta = 2$, and choose $\psi_{\lambda}(x) := \log ( 1 + x^2 / 2 )$, which satisfies \eqref{eq::PsiLambda}, the log-truncated loss becomes 
\begin{align}\label{eq::logtruncated}
	\ell_{\psi_{\lambda}, s, L_{\mathrm{abs}}}(y, f(x)) = s \cdot \log\left(1 + \frac{(y - f(x))^2}{2s^2}\right),
\end{align}
which is the Cauchy loss given in \eqref{equ::cauchyloss} with $\sigma = s$. 
Compared to the Cauchy loss, the log-truncated loss introduces additional flexibility through the parameter $\beta$, allowing fine-tuning for different noise distributions or robustness needs. However, this added flexibility comes at the cost of increased complexity in hyperparameter tuning and reduced interpretability.

\subsection{Kernel Cauchy Ridge Regression} \label{sec::KCRRmethod}

In this paper, we explore a kernel-based regressor that minimizes the Cauchy loss function to address the robust regression problem. Kernel-based regression is a non-parametric technique that leverages kernel functions to model the relationship between response and explanatory variables. Specifically, let $H$ represent the \textit{reproducing kernel Hilbert space} (\textit{RKHS}) induced by the Gaussian kernel function $k(x,x') := \exp(-\|x - x'\|_2^2 / \gamma^2)$ for $x, x' \in \mathcal{X}$, where $\gamma$ is the bandwidth parameter. For any regressor $f \in H$ and a clipping parameter $M > 0$, the clipped regressor is defined as 
\begin{align} \label{eq::clipping}
	\wideparen{f}(x) = 
	\begin{cases}
		-M & \text{ if } f(x) < -M,
		\\
		f(x) & \text{ if } f(x) \in [-M, M],
		\\
		M & \text{ if } f(x) > M.
	\end{cases}
\end{align}
The parameter $M$ may vary depending on the sample size $n$. By selecting $M$ larger than $\|f^*\|_{\infty}$ (the true regression function's range), the clipped regressor $\wideparen{f}$ is able to fully capture the behavior of $f^*$, thereby avoiding underestimation and preserving the accuracy of the regression model. The clipping operation ensures the regressor $\wideparen{f}$ does not produce extreme or nonsensical values, especially in the presence of unbounded noise or outlier data.

The following lemma demonstrates that the Cauchy loss is clippable, meaning that the Cauchy risk does not degrade after applying clipping, provided that the clipping parameter $M \geq \|f^*\|_{\infty}$.

\begin{lemma}\label{lem::validateclip}
	Let Assumptions \ref{ass::logmoment}, \ref{ass::symmetry}, and \ref{ass::decreasingtails} hold, and let $L$ denote the Cauchy loss function with $f^*$ as the true regression function. Additionally, assume the clipping parameter $M \geq \|f^*\|_{\infty}$. Then, for any regressor $f$, we have
	$\mathcal{R}_{L,P}(\wideparen{f}) \leq  \mathcal{R}_{L,P}(f)$.
\end{lemma}

As demonstrated in Lemma \ref{lem::validateclip}, the Cauchy loss satisfies the clipping condition \citep{steinwart2006oracle} under Assumptions \ref{ass::logmoment}, \ref{ass::symmetry} and \ref{ass::decreasingtails}. Specifically, for any regressor $f$, the risk of the clipped regressor $\wideparen{f}$ is no greater than that of the original regressor $f$, i.e., $\mathcal{R}_{L,P}(\wideparen{f}) \leq \mathcal{R}_{L,P}(f)$. This shows that applying the clipping operation to a regressor can result in a reduced or unchanged Cauchy risk.

Given a regularization parameter $\lambda > 0$, the \textit{kernel Cauchy ridge regressor} (\textit{KCRR}), denoted by $\wideparen{f}_D$, is obtained by minimizing the empirical Cauchy risk, as follows:
\begin{align}\label{eq::KCRR}
	f_D \in \argmin_{f \in H} \ \lambda \|f\|_H^2 + 
	\mathcal{R}_{L,\mathrm{D}}(\wideparen{f}).
\end{align}
Here, the first term $\lambda \|f\|_H^2$ controls the model complexity and the second term $\mathcal{R}_{L,\mathrm{D}}(\wideparen{f})$ represents the clipped empirical Cauchy risk, ensuring robustness against outliers. Thus, KCRR belongs to the class of \textit{clipped regularized empirical risk minimization} (\textit{CR-ERM}) \cite[Definition 7.18]{steinwart2008support}, which enhances traditional \textit{regularized empirical risk minimization} (\textit{RERM}) by applying a clipping operation. This clipping reduces the influence of outliers or noise, resulting in improved generalization properties. In KCRR, the robustness is further strengthened by integrating the clipping mechanism with the robust Cauchy loss function.

\section{Main Results and Statements} \label{sec::mainresults}

In this section, we aim to establish the convergence rates of KCRR in terms of $L_2$-risk. The $L_2$-risk measures the $L_2$-norm distance between $\wideparen{f}_D$ and the true regression function $f^*$, denoted as $\|\wideparen{f}_D - f^*\|_{L_2(P_X)}$, where $P_X$ is the marginal distribution of $P$ on $X$. This metric directly quantifies the mean squared distance between the regressor $\wideparen{f}_D$ and the regression function $f^*$. 
One of the key advantages of the $L_2$-risk metric is that it only depends on the distribution of $X$ and, therefore, that it is independent of the noise distribution, i.e.\ the distribution of $\epsilon$. This property makes it particularly suitable for robust regression problems. Unlike other evaluation metrics that rely on residuals, such as the Huber loss or absolute loss, $L_2$-risk does not require assumptions about the distribution of the noise or responses. Consequently, $L_2$-risk offers a general approach to assess the accuracy of the regressor across various noise distributions.
Moreover, the $L_2$-risk enables consistent comparisons across different models and methods, which is why it is widely used in robust regression contexts, as noted in \cite{feng2015learning}.

In Section \ref{sec::CalibrationInequalities}, we establish calibration inequalities specifically tailored for KCRR, laying the groundwork for understanding its performance. Section \ref{sec::RatesCauchy} delves into the excess risk bounds of KCRR in terms of the Cauchy loss, which is crucial for assessing the robustness of the method against outliers. Following this, Section \ref{sec::RatesL2} leverages the derived generalization bounds to calculate the convergence rates of KCRR concerning the $L_2$-risk, providing insights into the effectiveness of the regression approach. Lastly, Section \ref{sec::ComparisonRates} offers a comparative analysis of these convergence rates against those of existing methods, highlighting the advantages of KCRR within the broader context of robust regression techniques.

\subsection{Calibration Inequality for KCRR} \label{sec::CalibrationInequalities}

To establish the error bound of KCRR with respect to the \(L_2\)-risk, we first introduce a theorem that links the \(L_2\)-risk to the excess Cauchy risk. This fundamental relationship provides the basis for a more comprehensive analysis of KCRR's performance.

\begin{theorem} \label{lem::rell2Cauchylargesigma}
	Let Assumptions \ref{ass::logmoment}, \ref{ass::symmetry}, and \ref{ass::decreasingtails} hold. Additionally, let $\sigma > 0$ be the scale parameter of the Cauchy loss $L$ and let $f^*$ denote the true regression function. For any $f : \mathcal{X} \to \mathcal{Y}$, let $\wideparen{f}$ be the clipped version of $f$ defined as in \eqref{eq::clipping}  for a parameter $M \geq \|f^*\|_{\infty}$. Under these conditions, there exists a constant $c_1 > 0$, independent of $M$, such that for any $\sigma \geq 4M \vee c_1$, the following relationship holds:
	\begin{align}\label{eq::relation}
		\mathcal{R}_{L,P}(\wideparen{f}) - \mathcal{R}_{L,P}^* 
		\leq \|\wideparen{f} - f^*\|^2_{L_2(P_X)} 
		\leq 8\bigl(\mathcal{R}_{L,P}(\wideparen{f}) - \mathcal{R}_{L,P}^* \bigr).
	\end{align}
\end{theorem}

Theorem \ref{lem::rell2Cauchylargesigma} demonstrates that when $\sigma$ is sufficiently large, specifically for $\sigma \geq 4M \vee c_1$, the convergence rate of $\wideparen{f}$ with respect to the $L_2$-risk aligns with that of the excess Cauchy risk. This equivalence implies that minimizing the Cauchy risk is effectively the same as minimizing the $L_2$-risk. Consequently, this relationship helps to explain why the Cauchy loss is effective for regression in accurately learning the true regression function $f^*$ in the presence of heavy-tailed noise.

\subsection{Generalization Bound of KCRR  with respect to the Cauchy Risk} \label{sec::RatesCauchy}

Before presenting the generalization bound of KCRR with respect to the Cauchy risk, it is essential to introduce a smoothness assumption for the true regression function. This assumption is commonly applied in the analysis of non-parametric regression problems, as seen in works such as \cite{hardle1985optimal}.

\begin{assumption}[\textbf{H\"{o}lder Continuity}]\label{ass::Holder}
	Assume that the regression function $f^*: \mathbb{R}^d \to \mathbb{R}$ in the model \eqref{equ::regressionmodel} is \textit{$\alpha$-Hölder continuous} for some $\alpha \in (0, 1]$. This means that there exists a constant $c_{\alpha} \in (0, \infty)$ such that  $|f^*(x) - f^*(x')| \leq c_{\alpha} \| x - x' \|^{\alpha}$ holds for all $x, x' \in \mathbb{R}^d$. 
\end{assumption}

The following theorem establishes the generalization bound for KCRR when the scale parameter $\sigma$ of the Cauchy loss is large, specifically in the context of robust regression.

\begin{theorem} \label{thm::rates}
	Let Assumptions \ref{ass::logmoment}, \ref{ass::symmetry}, \ref{ass::decreasingtails}, and \ref{ass::Holder} hold, with $\sigma > 0$ as the scale parameter of the Cauchy loss $L$, and let $f^*$ be the true regression function. Additionally, let $\wideparen{f}_D$ denote the KCRR defined by \eqref{eq::KCRR}, where the clipping parameter $M$ satisfies $M \geq \|f^*\|_{\infty}$. Then, there exists a constant $c_1 > 0$ such that for any $\sigma \geq 4M \vee c_1$ and any $q \in (0,1)$, with probability at least $1 - 1/n$, the following holds
	\begin{align}\label{eq::Cauchyerrorbound}
		\lambda\|f_D\|_H^2 + \mathcal{R}_{L,P}(\wideparen{f}_D) - \mathcal{R}_{L,P}^* 
		\lesssim \lambda \gamma^{-d} + \gamma^{2\alpha} 
		+ \sigma^2 \bigl( n^{-1+q} + \lambda^{-q} \gamma^{-d} n^{-1} \bigr).
	\end{align}
\end{theorem}

The condition $\sigma \geq 4M \vee c_1$ in Theorem \ref{thm::rates} stems from Theorem \ref{lem::rell2Cauchylargesigma}, which indicates that the generalization bound in \eqref{eq::Cauchyerrorbound} is valid only for sufficiently large $\sigma$. However, from \eqref{eq::Cauchyerrorbound}, it is clear that the generalization bound with respect to the Cauchy loss decreases as $\sigma$ becomes smaller, suggesting that a smaller $\sigma$ is more desirable. Thus, $\sigma$ plays a critical role in 
improving the generalization performance while
ensuring the validity of the bound. Consequently, the optimal convergence rate is achieved when $\sigma \asymp 4M \vee c_1$. Note that in \eqref{eq::Cauchyerrorbound}, the exponent $q$, defined as $d/N$, originates from the upper bound on the entropy numbers of the Gaussian kernel in Lemma \ref{lem::entropygaussian}, where $N$ denotes the order of continuous differentiability of the Gaussian kernel $k$. As $N$ can be arbitrarily large for the Gaussian kernel, $q$ can be chosen arbitrarily close to zero. Similar small exponents $q$ also appear in generalization bounds of kernel methods such as Theorem 7.1 of \cite{mona2013optimal}.

By combining Theorems \ref{lem::rell2Cauchylargesigma} and \ref{thm::rates}, we can derive an upper bound on the $L_2$-risk of KCRR and observe that setting $\sigma \asymp 4M \vee c_1$ yields the fastest possible convergence rate in terms of $L_2$-risk.
On the one hand, the convergence rate of the KCRR regressor in terms of $L_2$-risk is guaranteed to match the convergence rate with respect to the Cauchy loss only when $\sigma$ exceeds the threshold $4M \vee c_1$. Since the $L_2$-risk is always greater than or equal to the Cauchy risk, this equivalence represents the best possible outcome, implying that increasing $\sigma$ beyond this threshold does not yield further improvement.
On the other hand, as previously discussed, choosing $\sigma \asymp 4M \vee c_1$ ensures the optimal convergence rate with respect to the Cauchy loss.
Thus, by balancing these two factors, it becomes clear that $\sigma \asymp 4M \vee c_1$ is the optimal choice for achieving the best $L_2$-risk error bound.

\subsection{Convergence Rates of KCRR with respect to the $L_2$-Risk}\label{sec::RatesL2}

In this section, we derive the convergence rate of KCRR with respect to the $L_2$-risk under two distinct scenarios: the practical case where the infinity norm $\|f^*\|_{\infty}$ is unknown, and the ideal case where $\|f^*\|_{\infty}$ is known to be bounded by a constant $M_0$. By comparing these two settings, we provide insights into how the availability of information about $\|f^*\|_{\infty}$ impacts the convergence behavior of KCRR.

First, in practical applications, the infinity norm of the regression function $f^*$ is typically unknown. To ensure that the clipping parameter $M \geq \|f^*\|_{\infty}$, we choose $M$ as a polynomial function of the sample size, i.e., $M \asymp n^p$, where $p > 0$. This guarantees that for sufficiently large $n \geq \|f^*\|_{\infty}^{1/p}$, the condition $M \geq \|f^*\|_{\infty}$ is always satisfied.

The following theorem presents the convergence rates of KCRR with respect to the $L_2$-risk in the realistic case where $\|f^*\|_{\infty}$ is unknown.

\begin{theorem} \label{thm::ratesanyM}
	Let Assumptions \ref{ass::logmoment}, \ref{ass::symmetry}, \ref{ass::decreasingtails}, and \ref{ass::Holder} hold. Additionally, let the KCRR $\wideparen{f}_D$ be defined by \eqref{eq::KCRR}, with the clipping parameter set to $M = n^p$, where $p \in (0, 1/2)$. Then there exists a constant $c_1 > 0$ such that for any $q \in (0, 1)$, if we select the parameters as follows:
	\begin{align*}
		\lambda \asymp n^{-\frac{1-2p}{1+q}}, 
		\qquad 
		\gamma \asymp n^{-\frac{1-2p}{(2\alpha+d)(1+q)}}, 
		\qquad 
		\sigma \asymp n^p,
	\end{align*}
	then, for any $n \geq (\|f^*\|_{\infty} \vee c_1)^{1/p}$, with probability at least $1 - 1/n$, the following holds 
	\begin{align}\label{eq::rateanyM}
		\|\wideparen{f}_D - f^*\|^2_{L_2(P_X)}  
		\lesssim n^{-\frac{2\alpha(1-2p)}{(2\alpha+d)(1+q)}}.
	\end{align}
\end{theorem}

Theorem \ref{thm::ratesanyM} establishes that for sufficiently large $n \geq (\|f^*\|_{\infty} \vee c_1)^{1/p}$, KCRR can achieve a convergence rate of $n^{-2\alpha/(2\alpha + d)}$, up to an arbitrarily small order depending on $p$ and $q$. This rate aligns with the minimax lower bound $n^{-2\alpha/(2\alpha + d)}$ for robust regression under Cauchy noise, as demonstrated in \cite[Corollary 1]{zhao2023minimax}. Hence, the convergence rate of KCRR in terms of the $L_2$-risk is almost minimax-optimal, confirming its efficiency in this setting. Moreover, note that the minimax convergence rate under Gaussian noise in terms of the $L_2$-risk is shown to be of the order $n^{-2\alpha/(2\alpha + d)}$ in \cite[Theorem 3.2]{gyorfi2006distribution}. Therefore, the convergence rates that KCRR achieves under the generalized Cauchy noise are almost the same as the possibly best rates under Gaussian noise, which illustrates the robustness of Cauchy loss against heavy-tailed noise.

In the following theorem, we establish the convergence rates of KCRR in terms of $L_2$-risk under the ideal scenario where the infinity norm $\|f^*\|_{\infty}$ is known to be bounded by a constant $M_0$. In this case, we set the clipping parameter $M$ to the constant $M_0$.

\begin{theorem}\label{thm::L2rates}
	Let Assumptions \ref{ass::logmoment}, \ref{ass::symmetry}, \ref{ass::decreasingtails}, and \ref{ass::Holder} hold, and suppose the true regression function $f^*$ is bounded by $M_0$, i.e., $\|f^*\|_{\infty} \leq M_0$. Furthermore, let the KCRR $\wideparen{f}_D$ be defined by \eqref{eq::KCRR}, using the clipping parameter $M = M_0$.  Then, there exists a constant $c_1 > 0$ such that for any $q \in (0,1/2)$, by choosing 
	\begin{align*}
		\lambda \asymp n^{-1}, 
		\qquad
		\gamma \asymp n^{-\frac{1}{2\alpha+d}}, 
		\qquad
		\sigma = 4 M_0 \vee c_1,
	\end{align*} 
	we obtain
	\begin{align*}
		\|\wideparen{f}_D - f^*\|^2_{L_2(P_X)} \lesssim n^{-\frac{2\alpha}{(2\alpha+d)(1+q)}}
	\end{align*}
	with probability $P^n$ at least $1-1/n$.
\end{theorem}

Theorem \ref{thm::L2rates} demonstrates that, in the ideal case, KCRR achieves the optimal convergence rate of $n^{- 2\alpha / (2\alpha + d)}$ in terms of the $L_2$-norm for any $n$, up to an arbitrarily small order related to $q$. This result highlights the strong robustness of the Cauchy loss, as it allows KCRR to maintain near-optimal performance even under challenging noise distributions.

\subsection{Comparison with Existing Convergence Rates} \label{sec::ComparisonRates}

In this section, we compare the convergence rates proposed in our study with those established in the existing literature. In Section \ref{sec::ComparisonFeng}, we focus on the convergence rates derived from the application of correntropy loss \citep{feng2015learning}. We evaluate these rates specifically in terms of the $L_2$-norm, which is a crucial metric for measuring the mean squared distance between estimated regressors and the true regression function. 
In Section \ref{sec::ComparisonXu}, we analyze the convergence rates associated with log-truncated absolute loss \citep{xu2023non}. This section will involve a detailed assessment of these rates in relation to traditional absolute loss metrics. 
Through these comparisons, we seek to provide a comprehensive overview of how our proposed convergence rates align with or differ from established methodologies, ultimately contributing to a more nuanced understanding of their applicability and performance in various regression scenarios.

\subsubsection{Comparison with Rates Concerning the Correntropy Loss} \label{sec::ComparisonFeng}

In the context of robust regression, \cite{feng2015learning} explore \textit{empirical risk minimization} (\textit{ERM}) concerning the correntropy loss defined in \eqref{equ::correntropyloss} within a general hypothesis space $\mathcal{F}$. Specifically, they formulate the problem as 
\begin{align}\label{eq::correntropy}
	f_{\mathrm{corr}} \in \argmin_{f \in \mathcal{F}} \mathcal{R}_{L,\mathrm{D}}(f).
\end{align}

Under the finite fourth-order moment assumption for the response distribution \citep[Assumption 3]{feng2015learning}, Theorem 4 in \cite{feng2015learning} establishes the following error bounds:
\begin{align}\label{eq::corrbound}
	\|f_{\mathrm{corr}} - f^*\|_{L_2(P_X)}^2 \lesssim \inf_{f \in \mathcal{F}} \|f - f^*\|_{L_2(P_X)}^2 + C_{\mathcal{F},P} \log(2/\delta) (\sigma^{-2} + \sigma n^{-1/(1+p)}),
\end{align}
with probability at least $1-\delta$, where $\delta \in (0,1)$ and the constant $p>0$ satisfies the Complexity Assumption I in \cite{feng2015learning}, i.e., the covering number of the function space $\mathcal{F}$ is bounded by $\log \mathcal{N}(\mathcal{F}, \eta) \leq c_{1,p} \eta^p$ with the constant $c_{1,p}>0$.

Let $B_H$ denote the unit ball of the RKHS $H$. Assume that the true regression function $f^*$ satisfies the Hölder continuity in Assumption \ref{ass::Holder}. 
According to Theorem 2, Theorem 3, and Inequality (1) in \cite{mona2013optimal}, there exists a function $f_0 \in \mathcal{F} := r B_H$ with $r \asymp \gamma^{-d/2}$ such that 
\begin{align}\label{eq::appcorr}
	\inf_{f \in \mathcal{F}} \|f - f^*\|_{L_2(P_X)}^2 
	\leq \|f_0 - f^*\|_{L_2(P_X)}^2 
	\lesssim \gamma^{2\alpha}.
\end{align}

Using the definition of $C_{\mathcal{F},P}$ from \cite[Lemma 7]{feng2015learning} and Lemma 4.23 in \cite{steinwart2008support}, we get
\begin{align}\label{eq::CFP}
	C_{\mathcal{F},P} 
	\lesssim \sup_{f\in \mathcal{F}} \|f\|_{\infty}^4 
	\leq \sup_{f\in \mathcal{F}} \|f\|_H^4 
	\leq r^4 
	\asymp \gamma^{-2d}.
\end{align}
Moreover, by the inequality (6.19) in \cite{steinwart2008support} and since the Gaussian kernel is infinitely often differentiable, the Complexity Assumption I in \cite{feng2015learning} holds for any arbitrarily small $p>0$ close to zero. By substituting \eqref{eq::CFP} and \eqref{eq::appcorr} into \eqref{eq::corrbound}, and choosing $\delta \asymp n^{-1}$, $\gamma \asymp n^{-1/(3(\alpha+d)(1+p))}$ and $\sigma \asymp n^{1/(3(1+p))}$, we obtain
\begin{align*}
	\|f_{\mathrm{corr}} - f^*\|_{L_2(P_X)}^2 
	\lesssim \gamma^{2\alpha} + \gamma^{-2d} \log n(\sigma^{-2} + \sigma n^{-1/(1+p)}) 
	\lesssim n^{-\frac{2\alpha}{3\alpha+3d} + p'},
\end{align*}
with probability at least $1-1/n$, where $p'$ is an arbitrarily small positive number related to $p$. This result indicates that the convergence rate of $f_{\mathrm{corr}}$ is significantly slower than our convergence rate derived in \eqref{eq::rateanyM}. Introducing the clipping operation and our refined analysis of the calibration inequality in \eqref{eq::relation} may improve the convergence rates of $f_{\mathrm{corr}}$ and relax the moment assumption for the noise.

\subsubsection{Comparison with Rates Concerning the Log-Truncated Loss} \label{sec::ComparisonXu}

More recently, \cite{xu2023non} investigated \textit{regularized empirical risk minimization} (\textit{RERM}) with respect to the log-truncated absolute loss $\ell_{\psi_{\lambda},s,L_{\mathrm{abs}}}$ in \eqref{eq::logtruncated} for $\lambda(t):=|t|^{\beta}$ with $\beta \in (1,2)$ within a general hypothesis space $\mathcal{F}$. The optimization problem is formulated as 
\begin{align} \label{eq::RERMMLAL}
	\widehat{f}_{\beta} := \argmin_{f\in \mathcal{F}} 
	\rho \|f\|^2 +
	\mathcal{R}_{\ell_{\psi_{\lambda},s,L_{\mathrm{abs}}},D}(f),
\end{align}
where $\rho > 0$ is the regularization parameter.
Under Assumption (Q.2) in \cite{xu2023non}, which states that $\sup_{f\in\mathcal{F}} \mathbb{E} \lambda(L_{\mathrm{abs}}(Y,f(X))) < \infty$, it follows that $\mathbb{E} |\epsilon|^{\beta} < \infty$. Utilizing the proof techniques from \cite[Theorem 12]{xu2023non}, an upper bound on the excess absolute risk can be established as 
\begin{align}\label{eq::fbetaerror}
	\mathcal{R}_{L_{\mathrm{abs}}, P}(\widehat{f}_{\beta}) -  \mathcal{R}_{L_{\mathrm{abs}}, P}^* &\leq 2\delta \mathbb{E} D_{Y,X} + 2^{\beta-1} s^{\beta-1} \sup_{f\in \mathcal{F}} \mathbb{E} \lambda(L_{\mathrm{abs}}(Y, f(X))) + 2^{\beta-1} s^{\beta-1} \delta^{\beta}  \mathbb{E} \lambda(D_{Y,X})
	\nonumber\\
	& \phantom{=}
	+ \frac{\log(1/\delta)}{ns} + \frac{\log(\mathcal{N}(\mathcal{F},\delta)/\delta)}{ns} + s^{\beta-1} \inf_{f\in\mathcal{F}} \mathbb{E} \lambda(L_{\mathrm{abs}}(Y, f(X))) 
	\nonumber\\
	& \phantom{=}
	+ \inf_{f\in \mathcal{F}} |\mathcal{R}_{L_{\mathrm{abs}}, P}(f) -  \mathcal{R}_{L_{\mathrm{abs}}, P}^*| + \rho \sup_{f\in\mathcal{F}} \|f\|_{\mathcal{F}}^2,
\end{align}
with probability at least $1-2\delta$, where $D_{Y,X}$ satisfies 
\begin{align*}
	L_{\mathrm{abs}}(y, f(x)) - L_{\mathrm{abs}}(y, f'(x)) \leq D_{Y,X}\|f - f'\|_{\infty}.
\end{align*}

By applying Theorem 2, Theorem 3, and Inequality (1) from \cite{mona2013optimal}, we find that there exists a function $f_0 \in H$ such that
\begin{align}\label{eq::approxerrorLab}
	\|f_0\|_H \lesssim \gamma^{-d/2} 
	\qquad
	\text{and} 
	\qquad
	\mathcal{R}_{L_{\mathrm{abs}}, P}(f_0) -  \mathcal{R}_{L_{\mathrm{abs}}, P}^* \lesssim \gamma^{\alpha}. 
\end{align}
To encompass $f_0$, we define the hypothesis space $\mathcal{F} := c \gamma^{-d/2} B_H$, where $c$ is a constant ensuring $f_0 \in \mathcal{F}$. Utilizing Theorem 6.27 and Lemma 6.21 from \cite{steinwart2008support}, we can estimate the covering number as 
\begin{align}\label{eq::CoverNumber}
	\log \mathcal{N}(\mathcal{F}, \delta) \lesssim \gamma^{-d-pd/2} \delta^{-p}, 
	\qquad
	p \in (0,1). 
\end{align}
Next, we substitute $D_{y,x} = 1$ for the absolute loss, along with \eqref{eq::approxerrorLab} and \eqref{eq::CoverNumber}, into \eqref{eq::fbetaerror}. By choosing 
\begin{align*}
	\rho \lesssim n^{-\frac{(\alpha+d)(\beta-1)}{\alpha\beta + (\beta-1)d}}, 
	\quad 
	\gamma \asymp n^{-\frac{\beta-1}{\alpha\beta + (\beta-1)d}}, 
	\quad 
	\delta \asymp n^{-\frac{\alpha(\beta-1)}{\alpha\beta + (\beta-1)d}}, 
	\quad 
	s \asymp n^{-\frac{\alpha}{\alpha\beta + (\beta-1)d}},
	\quad
	\delta = 1/(2n),
\end{align*}
we obtain that for any $p\in (0,1)$, 
\begin{align} \label{eq::CRLAL}
	\mathcal{R}_{L_{\mathrm{abs}}, P}(\widehat{f}_{\beta}) -  \mathcal{R}_{L_{\mathrm{abs}}, P}^* \lesssim n^{-\frac{\alpha(\beta-1)}{\alpha\beta+(\beta-1)d}+p}
\end{align}
holds with probability at least $1-1/n$.

To facilitate comparison, we will derive the convergence rate of our KCRR defined in \eqref{eq::KCRR} with respect to the absolute loss using Theorem \ref{thm::ratesanyM}. 
Given two random variables $Z$ and $Z'$, it follows that 
\begin{align*}
	\mathbb{E}(|Z| - |Z'|) \leq \mathbb{E}|Z - Z'| 
	\leq (\mathbb{E}|Z - Z'|^2)^{1/2}.
\end{align*}
From this, we can deduce that 
\begin{align*}
	\mathcal{R}_{L_{\mathrm{abs}}, P}(f) - \mathcal{R}_{L_{\mathrm{abs}}, P}^* 
	\leq \|f - f^*\|_{L_1(P_X)} 
	\leq \|f - f^*\|_{L_2(P_X)}.
\end{align*}
Applying Theorem \ref{thm::ratesanyM}, we obtain that for any $p \in (0, 1/2)$ and $q \in (0, 1)$, 
\begin{align}\label{eq::ourrateabs}
	\mathcal{R}_{L_{\mathrm{abs}}, P}(\wideparen{f}_D) - \mathcal{R}_{L_{\mathrm{abs}}, P}^* 
	\lesssim n^{-\frac{\alpha(1-2p)}{(2\alpha + d)(1+q)}},
\end{align}
with probability at least $1-1/n$. This rate in \eqref{eq::ourrateabs} is faster than the convergence rate presented in \eqref{eq::CRLAL} for any $\beta \in [1, 2)$, as $p$ and $q$ can be chosen to be arbitrarily small.
Furthermore, our result in \eqref{eq::ourrateabs} holds under the assumption that $\mathcal{R}_{L_{\mathrm{abs}}, P}(\wideparen{f})$ exists for any $f$, which requires only that $\mathbb{E} |\varepsilon| < \infty$. This condition is less stringent than the finite $\beta$-th moment assumption of $\mathbb{E} |\epsilon|^{\beta} < \infty$ with $\beta \in (1, 2)$.

More importantly, the theoretical results under \cite[Assumption (Q.2)]{xu2023non} do not hold for noise distributions that lack a finite absolute mean, such as Cauchy noise and Pareto noise. Furthermore, their error analysis cannot be applied or extended to these types of noise, as the assumption is crucial for their analysis.
Specifically, to bound the difference between the risk associated with the log-truncated absolute loss $\ell_{\psi_{\lambda},s,L_{\mathrm{abs}}}$ and the absolute loss $L_{\mathrm{abs}}$, the authors utilize Markov's inequality along with \eqref{eq::PsiLambda}. In \cite[Lemmas 17 and 18]{xu2023non}, they demonstrate that the difference between the exponential transformations of these two losses can be bounded by terms that depend on $\mathbb{E} \lambda(L_{\mathrm{abs}}(Y,f(X))/s)$. This expectation is finite for any $f \in \mathcal{F}$ only when \cite[Assumption (Q.2)]{xu2023non} is satisfied.

\section{Error Analysis} \label{sec::ErrorAnalysis}

In this section, we conduct an error analysis to establish the generalization error bound of KCRR, specifically in terms of $L_2$-risk. One of the primary challenges in error analysis for robust regression lies in the fact that the regressor is fitted by minimizing a robust loss function. However, our goal is to derive the generalization error bound with respect to a different loss or risk criterion.
In our study, we utilize the Cauchy loss to train our KCRR regressor. Nevertheless, it is essential to establish the convergence rate for KCRR based on the $L_2$-risk, which serves as an important performance metric. This analysis will provide a comprehensive understanding of how well KCRR performs under this alternative criterion, contributing to our broader investigation of its effectiveness in various regression scenarios.

First, we note that the $L_2$-risk of the KCRR regressor $\wideparen{f}_D$ can be expressed as 
\begin{align}\label{eq::newdecomp}
	\|\wideparen{f}_D - f^*\|^2_{L_2(P_X)} &= \frac{\|\wideparen{f}_D - f^*\|^2_{L_2(P_X)}}{\mathcal{R}_{L, P}(\wideparen{f}_D) - \mathcal{R}_{L, P}^*} \cdot (\mathcal{R}_{L, P}(\wideparen{f}_D) - \mathcal{R}_{L, P}^*)
	\nonumber\\
	&\leq \sup_{f:\mathcal{X} \to \mathbb{R}} \frac{\|\wideparen{f} - f^*\|^2_{L_2(P_X)}}{\mathcal{R}_{L, P}(\wideparen{f}) - \mathcal{R}_{L, P}^*} \cdot (\mathcal{R}_{L, P}(\wideparen{f}_D) - \mathcal{R}_{L, P}^*).
\end{align}
In this expression, the first term is the ratio of the $L_2$-risk to the excess Cauchy risk, which will be further explored in Section \ref{sec::relationerror}. The second term represents the excess Cauchy risk of KCRR. According to Inequality (7.39) in \cite{steinwart2008support}, this can be decomposed into a sample error term and an approximation error term:
\begin{align}\label{eq::Cauchydecomp}
	\mathcal{R}_{L, P}(\wideparen{f}_D) - \mathcal{R}_{L, P}^*
	\leq 
	2 \sup_{f \in H} |\mathcal{R}_{L,D}(\wideparen{f}) - \mathcal{R}_{L,P}(\wideparen{f})|
	+ \inf_{f\in H} (\lambda \|f\|_H^2 + \mathcal{R}_{L,P}(\wideparen{f}) - \mathcal{R}_{L,P}^*).
\end{align}
These two error terms will be analyzed in detail in Sections \ref{sec::sampleerror} and \ref{sec::approxerror}, respectively.

It is important to note that once the general relationship between $L_2$-risk and the excess Cauchy risk is established, the expression in \eqref{eq::newdecomp} effectively transforms the analysis of the $L_2$-risk of KCRR into an analysis of the excess Cauchy risk of KCRR. Since KCRR is the \textit{clipped regularized empirical risk minimizer} (\textit{CR-ERM}) with respect to the Cauchy loss, we can directly derive an error bound for the excess Cauchy risk of KCRR within the CR-ERM framework. Moreover, by leveraging the supremum risk ratio, \eqref{eq::newdecomp} allows us to bypass the differences between the two risks of KCRR, streamlining the analysis process.

\subsection{Relationship between the Excess Cauchy Risk and $L_2$-Risk} \label{sec::relationerror}

In this section, we outline the key concepts involved in studying the relationship between the excess Cauchy risk and the $L_2$-risk. Specifically, we examine this relationship for a large scale parameter $\sigma$ in Theorem \ref{lem::rell2Cauchylargesigma}.

\vspace{2mm}
\noindent
\textbf{Proof Sketch} [of Theorem \ref{lem::rell2Cauchylargesigma}]  
Let $\xi(x) := \wideparen{f}(x) - f^*(x)$. Then the $L_p$-risk can be expressed as $\|\xi\|_{L_p(P_X)}$ for $1 \leq p \leq \infty$. It is straightforward to establish that the $L_{\infty}$-risk $\|\xi\|_{\infty}$ can be upper-bounded by
\begin{align}\label{eq::bound2M}
	\|\xi\|_{\infty} \leq \|\wideparen{f}\|_{\infty} + \|f^*\|_{\infty} \leq 2M.
\end{align}
By elementary analysis, we can show that when $\sigma \geq 4M$, for any $\epsilon$ and $x$, 
\begin{align*}
	\frac{|(\epsilon - \xi(x))^2 - \epsilon^2|}{\epsilon^2 + \sigma^2} 
	\leq \frac{2}{3}.
\end{align*}
Then, for the Cauchy loss $L$, we have
\begin{align}\label{eq::simplify}
	&\mathcal{R}_{L,P}(\wideparen{f}) - \mathcal{R}_{L,P}^*
	\nonumber\\
	&= \mathbb{E}_X \mathbb{E}_{Y|X} \bigl( L(Y, \wideparen{f}(X)) - L(Y, f^*(X)) \bigr) \qquad \qquad  \text{(By definition of $\mathcal{R}_{L,P}$ and Lemma \ref{lem::optimality}})
	\nonumber\\ 
	& = \mathbb{E}_X \mathbb{E}_{\epsilon} \sigma^2 \log \biggl( 1 + \frac{\big(\epsilon - \xi(X)\big)^2 - \epsilon^2}{\epsilon^2 + \sigma^2} \biggr) \qquad \qquad \text{(By definition of $L$ and $\epsilon = Y-f^*(X)$)}
	\nonumber\\
	& \geq \mathbb{E}_X \mathbb{E}_{\epsilon} \sigma^2 \biggl( \frac{(\epsilon - \xi(X))^2 - \epsilon^2}{\epsilon^2 + \sigma^2} - \frac{3((\epsilon - \xi(X))^2 - \epsilon^2)^2}{2 (\epsilon^2 + \sigma^2)^2} \biggr) 
	\;\; \text{($\log(1+t) \geq t-3t^2/2$, $t>-2/3$)}
	\nonumber\\
	&=  \mathbb{E}_X \mathbb{E}_{\epsilon} \sigma^2 \biggl( \frac{\xi(X)^2}{\epsilon^2 + \sigma^2} - \frac{3\xi(X)^4 + 12\epsilon^2 \xi(X)^2}{2 (\epsilon^2 + \sigma^2)^2} \biggr)
	\qquad \quad \text{(By the symmetry of $\epsilon$)}
	\nonumber\\
	& = \|\xi\|^2_{L_2(P_X)} - \|\xi\|^2_{L_2(P_X)} \mathbb{E}_{\epsilon} \biggl( \frac{7\epsilon^2\sigma^2 + \epsilon^4}{(\epsilon^2 + \sigma^2)^2} \biggr) - \|\xi\|^4_{L_4(P_X)} \mathbb{E}_{\epsilon} \biggl( \frac{3\sigma^2/2}{(\epsilon^2 + \sigma^2)^2} \biggr). 
\end{align} 
Applying the dominated convergence theorem, we can show that
\begin{align*}
	\lim_{\sigma \to \infty} \mathbb{E}_{\epsilon} \biggl( \frac{7 \epsilon^2 \sigma^2 + \epsilon^4}{(\epsilon^2 + \sigma^2)^2} \biggr) = 0. 
\end{align*}
Thus there exists a constant $c_1 > 0$ such that for any $\sigma \geq c_1$, 
\begin{align}\label{eq::TermNoise1}
	\mathbb{E}_{\epsilon} \biggl( \frac{7 \epsilon^2 \sigma^2 + \epsilon^4}{(\epsilon^2 + \sigma^2)^2} \biggr) \leq 1/2. 
\end{align}
Moreover, we have
\begin{align}\label{eq::TermNoise2}
	\mathbb{E}_{\epsilon} \biggl( \frac{3 \sigma^2 / 2}{(\epsilon^2 + \sigma^2)^2} \biggr)
	\leq \mathbb{E}_{\epsilon} \biggl( \frac{3 \sigma^2 / 2}{(\sigma^2)^2} \biggr)
	= \frac{3}{2 \sigma^2}.
\end{align}
Combining \eqref{eq::simplify} with \eqref{eq::TermNoise1} and \eqref{eq::TermNoise2}, we obtain
that for any $\sigma \geq c_1$, 
\begin{align*}
	\mathcal{R}_{L,P}(\wideparen{f}) - \mathcal{R}_{L,P}^* \geq \|\xi\|^2_{L_2(P_X)}/2 -  3\|\xi\|^4_{L_4(P_X)}/(2\sigma^2).
\end{align*}
Multiplying both sides by 2 and rearranging terms gives
\begin{align}\label{eq::relationanysigmaDelta}
	\|\xi\|^2_{L_2(P_X)} 
	\leq 2 \bigl( \mathcal{R}_{L,P}(\wideparen{f}) - \mathcal{R}_{L,P}^* \bigr) + 3\|\xi\|^4_{L_4(P_X)}/\sigma^2.
\end{align}
Therefore, by \eqref{eq::bound2M}, if $M/\sigma \leq 1/4$, then we have
\begin{align} \label{eq::L4L2}
	\|\xi\|_{L_4(P_X)}^4 /\sigma^2
	& \leq \|\xi\|_{L_2(P_X)}^2 \cdot \|\xi\|_{\infty}^2 / \sigma^2 
	\leq \|\xi\|_{L_2(P_X)}^2 \cdot (2M)^2 / \sigma^2
	\nonumber\\
	& = 4(M/\sigma)^2 \|\xi\|_{L_2(P_X)}^2
	\leq (1/4)\|\xi\|^2_{L_2(P_X)}. 
\end{align}
Combining \eqref{eq::L4L2} and \eqref{eq::relationanysigmaDelta}, we find
\begin{align*}
	\|\xi\|^2_{L_2(P_X)} 
	\leq 2\bigl( \mathcal{R}_{L,P}(\wideparen{f}) - \mathcal{R}_{L,P}^*\bigr) +
	(3/4)\|\xi\|^2_{L_2(P_X)}.
\end{align*}
From this inequality, we can rearrange terms to obtain
\begin{align}\label{eq::calibrationDelta}
	\|\xi\|^2_{L_2(P_X)} 
	\leq 8 \bigl( \mathcal{R}_{L,P}(\wideparen{f}) - \mathcal{R}_{L,P}^*\bigr).
\end{align}
This completes the proof sketch of the assertion in Theorem \ref{lem::rell2Cauchylargesigma}.

\subsection{Bounding the Sample Error} \label{sec::sampleerror}

In this section, we derive an upper bound for the sample error given by
\begin{align*}
	\sup_{f \in H} |\mathcal{R}_{L,D}(\wideparen{f}) - \mathcal{R}_{L,P}(\wideparen{f})|,
\end{align*}
which arises from both the randomness of the data and the complexity of the function space $H$. One significant analytical challenge in bounding this sample error is related to the unbounded response, which may not even possess a finite absolute mean under Assumption~\ref{ass::logmoment}.

The following lemma establishes the Lipschitz property of the Cauchy loss function, which is an important characteristic for assessing the stability and continuity of this loss with respect to changes in its inputs.

\begin{lemma}[Lipschitz Property]\label{lem::lip}
	Let $L$ denote the Cauchy loss function with $\sigma > 0$ as the scale parameter. Then, for any two functions $f$ and $g$ and any response $y$, we have 
	\begin{align*}
		|L(y, f(x)) - L(y, g(x))| 
		\leq \sigma \cdot |f(x) - g(x)|.
	\end{align*}
\end{lemma}

This property indicates that the difference in the Cauchy loss between the predictions $f(x)$ and $g(x)$ is bounded by the product of the scale parameter $\sigma$ and the absolute difference between the outputs of the two functions. The Lipschitz property is particularly beneficial in the context of convergence analysis, as it guarantees that small changes in model predictions lead to proportionately small changes in the loss.

By applying Lemma \ref{lem::lip} to the functions $f = \wideparen{f}_D$ and $g = f^*$, we can derive an upper bound for the excess Cauchy loss. According to the Lipschitz property established in Lemma \ref{lem::lip}, the excess Cauchy loss can be expressed as 
\begin{align*}
	|L(y, \wideparen{f}_D(x)) - L(y, f^*(x))| 
	\leq \sigma \cdot |\wideparen{f}_D(x) - f^*(x)|.
\end{align*}
Given that the supremum norm between the estimated function $\wideparen{f}_D$ and the optimal function $f^*$ satisfies $\|\wideparen{f}_D - f^*\|_{\infty} \leq \|\wideparen{f}_D\|_{\infty} + \|f^*\|_{\infty} \leq 2M$, we can substitute this supremum norm constraint into the inequality to obtain
\begin{align*}
	|L(y, \wideparen{f}_D(x)) - L(y, f^*(x))| 
	\leq \sigma \cdot (2M) 
	= 2M \sigma.
\end{align*}
This result demonstrates that the excess Cauchy loss can be effectively bounded by $2M \sigma$, even in scenarios where the response variable $y$ is unbounded.

Based on the variance bound definition presented in \cite[(7.36)]{steinwart2008support}, we derive the variance bound for the Cauchy loss by combining its Lipschitz property with the calibration inequality established in Theorem \ref{lem::rell2Cauchylargesigma}, as stated in the following lemma.

\begin{lemma}[Variance Bound]\label{lem::variancebound}
	Let Assumptions \ref{ass::logmoment}, \ref{ass::symmetry}, and \ref{ass::decreasingtails} hold. Let $L$ denote the Cauchy loss function with $\sigma > 0$ as the scale parameter and $f^*$ as the true regression function. Furthermore, let $f: \mathcal{X} \to \mathcal{Y}$ be a function and $\wideparen{f}$ be the clipped version of $f$, defined as in \eqref{eq::clipping}, with the clipping parameter $M$ satisfying $M \geq \|f^*\|_{\infty}$. Then, there exists a constant $c_1 > 0$ such that for any $\sigma > 4M \vee c_1$, the following holds
	\begin{align*}
		\mathbb{E}_P |L(Y, \wideparen{f}(X)) - L(Y, f^*(X))|^2 
		\leq 8 \sigma^2 \cdot \bigl(\mathcal{R}_{L,P}(\wideparen{f}) - \mathcal{R}_{L,P}^*\bigr).
	\end{align*}
\end{lemma}

This lemma establishes a quantitative relationship between the expected squared difference in Cauchy loss and the excess Cauchy risk. The significance of this variance bound lies in its ability to control the variance of excess Cauchy loss using its mean, even if $Y$ has an infinite absolute mean. This result not only highlights the importance of clipping to mitigate the effects of unbounded responses but also plays a crucial role in establishing generalization guarantees for KCRR with respect to the Cauchy loss.

Lemma \ref{lem::variancebound} shows that the Cauchy loss satisfies the variance bound defined by 
\begin{align*}
	\mathbb{E} h^2_f \leq V \mathbb{E} h_f^{\theta},
\end{align*}
where the excess loss $h_f(X, Y) := L(Y, \wideparen{f}(X)) - L(Y, f^*(X))$, with $V = 8\sigma^2$ and $\theta = 1$. The exponent $\theta$ lies in the range $[0,1]$, and $\theta = 1$ corresponds to Bernstein's condition \citep{van2015fast} and represents the optimal exponent for this variance bound. Consequently, Lemma \ref{lem::variancebound} ensures low variance whenever the Cauchy risk $\mathcal{R}_{L,P}(\wideparen{f})$ is minimized. This implies that decreasing the risk directly controlls the variance of the excess loss.

The calibration result in \eqref{eq::relation} (Theorem \ref{lem::rell2Cauchylargesigma}) is pivotal in achieving the variance bound with the optimal exponent $\theta = 1$. This insight can guide a general analysis framework for establishing variance bounds for other robust loss functions.

Overall, this variance bound shows that small excess risks lead to well-controlled excess loss variance, strengthening the theoretical foundation of KCRR. Moreover, it also lays the groundwork for a robust oracle inequality, highlighting its predictive reliability in practical applications.

\begin{proposition}[Oracle Inequality]\label{prop::oracle}
	Let Assumptions \ref{ass::logmoment}, \ref{ass::symmetry}, and \ref{ass::decreasingtails} hold. Let $L$ denote the Cauchy loss function with $\sigma > 0$ as the scale parameter and $f^*$ as the true regression function. 
	Furthermore, let $\wideparen{f}_D$ be the KCRR defined in \eqref{eq::KCRR} with the clipping parameter $M \geq \|f^*\|_{\infty}$. Then there exists a constant $c_1 > 0$ such that for any $f_0\in H$, $q \in (0,1)$, $\sigma > 4M \vee c_1$, and $\tau > 0$, the following inequality holds with probability at least $1 - 2 e^{- \tau}$,
	\begin{align}\label{eq::oracle}
		\begin{split}
			\lambda \|f_D\|_H^2 + \mathcal{R}_{L,P}(\wideparen{f}_D)  - \mathcal{R}_{L, P}^* 
			& \leq 6 \left( \lambda \|f_0\|_H^2 + \mathcal{R}_{L,P}(\wideparen{f}_0)  - \mathcal{R}_{L, P}^* \right) 
			\\
			& \phantom{=}
			+ 1216 c_2^2 \sigma^2 \left( \tau n^{-1} + \lambda^{-q} \gamma^{-d} n^{-1} \right)
		\end{split}
	\end{align}
	where $c_2$ is a constant only depending on $q$ and the data dimension $d$.
\end{proposition}

This result establishes a key relationship between the regularized excess risk, the approximation error, and the sample error bound. It demonstrates how the choice of parameters affects the performance of KCRR, offering a theoretical guarantee for its generalization ability with respect to the Cauchy loss. In the next section, we will provide an upper bound for the approximation error term $\lambda \|f_0\|_H^2 + \mathcal{R}_{L,P}(\wideparen{f}_0) - \mathcal{R}_{L,P}^*$, where $f_0 \in H$ is to be chosen, as outlined in \eqref{eq::oracle}.

\subsection{Bounding the Approximation Error}\label{sec::approxerror}

In this section, we derive an upper bound for the approximation error term
\begin{align*}
	\inf_{f\in H} \lambda \|f\|_H^2 + \mathcal{R}_{L,P}(\wideparen{f}) - \mathcal{R}_{L,P}^*
\end{align*}
in the context of \eqref{eq::oracle}. This bound provides valuable insights into the KCRR model’s ability to approximate the true regression function, emphasizing the relationship between model complexity, the risk of a selected function $\wideparen{f}_0$, and the optimal risk $\mathcal{R}_{L,P}^*$.

\begin{proposition}\label{prop::approx}
	Let Assumptions \ref{ass::logmoment}, \ref{ass::symmetry}, \ref{ass::decreasingtails}, and \ref{ass::Holder} hold. 
	Moreover, let $L$ be the Cauchy loss function and the clipping parameter $M \geq \|f^*\|_{\infty}$. Then there exists a regressor $f_0 \in H$ such that 
	\begin{align*}
		\lambda \|f_0\|_H^2 + \mathcal{R}_{L,P}(\wideparen{f}_0) - \mathcal{R}_{L,P}^* 
		\leq c_3 (\lambda \gamma^{-d} + \gamma^{2\alpha}),
	\end{align*}
	where $c_3$ is a constant only depending on the data dimension $d$, the H\"older exponent $\alpha$ and $\|f^*\|_{\infty}$.
\end{proposition}

This upper bound is critical as it underscores how well the function space $H$ can approximate the true regression function. It captures the balance between the regularization term $\lambda \|f_0\|_H^2$ and the excess risk of $\mathcal{R}_{L,P}(\wideparen{f}_0)$ relative to the optimal risk $\mathcal{R}_{L,P}^*$. The bandwidth $\gamma$ plays a key role in determining this bound, as a smaller $\gamma$ of the approximation function $f_0$ can reduce excess risks but impose a higher penalty. This result clarifies the factors influencing the approximation errors of our method.

\section{Experiments}\label{sec::experiments}

To evaluate the effectiveness of KCRR, we conduct a series of experiments using an iterative algorithm to solve KCRR and compare its performance against three well-established methods: \textit{kernel least absolute deviation} (\textit{KLAD}) \citep{wang2014least}, \textit{kernel-based Huber regression} (\textit{KBHR}) \citep{wang2022huber}, and the \textit{maximum correntropy criterion for regression} (\textit{MCCR}) \citep{feng2015learning}. Each of these models can be formulated as 
\begin{align*}
	f_D \in \argmin_{f \in H} \ \lambda \|f\|_H^2 + \mathcal{R}_{L,\mathrm{D}}(f),
\end{align*}
where $H$ represents the reproducing kernel Hilbert space (RKHS), and $L$ is the specific loss function used in each method: absolute loss for KLAD, Huber loss for KBHR, and correntropy loss for MCCR. This framework enables a consistent comparison of model performance across different loss functions.

\subsection{Solving KCRR}

In this section, we present an empirical approach to solving the KCRR problem defined in \eqref{eq::KCRR}. According to the representer theorem, the solution to KCRR lies within the span of the kernel functions, which can be expressed as 
\begin{align*}
	\biggl\{ f := \sum_{i=1}^n a_i k(\cdot , X_i) + b \ \bigg| \ b \in \mathbb{R}, \ a_i \in \mathbb{R}, \ i \in [n] \biggr\},
\end{align*}
where $a_i$ are the coefficients and $b$ is the intercept term.

To solve the KCRR problem, our objective is to determine the coefficients $a_i$ and intercept $b$ that minimize the following objective function:
\begin{align} \label{eq::tobesolved}
	\min_{a_i, b} \ L \biggl( Y_i, \sum_{j=1}^n a_j k(X_i, X_j) + b \biggr) + \lambda \sum_{i,j=1}^n a_i a_j k(X_i, X_j),
\end{align}
where the first term is the loss function applied to the predictions, and the second term is a regularization term scaled by $\lambda$, which controls the complexity of the solution. This formulation provides a balance between fitting the data and maintaining model simplicity.

To solve \eqref{eq::tobesolved}, we use the \textit{iterated reweighted least squares} (\textit{IRLS}) method. This method iteratively minimizes a weighted least-squares problem, defined as 
\begin{align}\label{eq::aitbt}
	(a_i^{(t)}, b^{(t)}) := \argmin_{a_i, b} \; w_i^{(t-1)}  \biggl( Y_i - \biggl( \sum_{j=1}^n a_j k(X_i, X_j) + b \biggr) \biggr)^2 + \lambda \sum_{i,j=1}^n a_i a_j k(X_i, X_j),
\end{align}
where $w_i^{(t-1)}$ are the weights from the previous iteration.

After each minimization step, the weights are updated as follows:
\begin{align}\label{eq::updateweights}
	w_i^{(t)} := \frac{L \left( Y_i, \sum_{j=1}^n a_j^{(t)} k(X_i, X_j) + b^{(t)} \right)}{\left( Y_i - \sum_{j=1}^n a_j^{(t)} k(X_i, X_j) - b^{(t)} \right)^2}.
\end{align}
Here, $w_i^{(t)}$ represents the ratio of the Cauchy loss to the squared loss, adjusting the weighting of residuals based on the updated estimates. This iterative process allows the model to handle outliers more robustly by dynamically reweighting each data point.

We begin by initializing the weights as $w_i^{(0)} = 1$. In each iteration, we first derive the explicit solution of \eqref{eq::aitbt} for $a_i^{(t)}$ and $b^{(t)}$. Following this, we update the weights $w_i^{(t)}$ based on the current solution using \eqref{eq::updateweights}. This iterative process continues until convergence.

Given that the KCRR model is non-convex, applying the IRLS method to solve \eqref{eq::tobesolved} guarantees convergence only to a stationary point \cite{aftab2015convergence}. Nevertheless, the IRLS method is known for its efficiency and stability in reaching stationary points, making it a practical and effective approach in non-convex settings. Empirical evidence also suggests that IRLS performs well for similar non-convex optimization problems, indicating that a stationary point solution is likely to meet the accuracy requirements of the KCRR model. Therefore, in light of its practical effectiveness and empirical validation, a stationary point solution adequately serves our objectives.

\subsection{Synthetic Experiments}

In this section, we present simulation experiments on robust regression to demonstrate the effectiveness of KCRR in managing heavy-tailed noise.

To evaluate our approach, we use Friedman’s benchmark functions \citep{friedman1991multivariate} as the regression function $f$. These functions are commonly applied in studies of robust regression, as seen in works like \cite{feng2015learning}. Below, we describe three of Friedman’s benchmark functions, which we use as the regression function $f$:
\begin{itemize}
	\item[$(I)$] 
	$f(x) := 10 \sin(x^{(1)} x^{(2)}) + 20(x^{(3)} - 0.5)^2 + 10x^{(4)} + 5x^{(5)}$, where $x := (x^{(j)})_{j\in[5]}$;
	\item[$(II)$] 
	$f(x) := \sqrt{(x^{(1)})^2 + (x^{(2)} x^{(3)} - 1/(x^{(2)} x^{(4)}))^2}$, where $x := (x^{(j)})_{j\in[4]}$;
	\item[$(III)$] 
	$f(x) := \arctan\left( ( x^{(2)} x^{(3)} - 1/(x^{(2)} x^{(4)}) ) / x^{(1)} \right)$, where $x := (x^{(j)})_{j\in[4]}$.
\end{itemize}
For regression function $(I)$, each coordinate $x^{(j)}$, $j \in [5]$, is independently drawn from a uniform distribution on the interval $[0, 1]$. For regression functions $(II)$ and $(III)$, each coordinate $x^{(j)}$, $j \in [4]$, is independently drawn from distinct uniform distributions over the following intervals: $x^{(1)} \sim U[0, 100]$, $x^{(2)} \sim U[40\pi, 500\pi]$, $x^{(3)} \sim U[0, 1]$, and $x^{(4)} \sim U[1, 11]$.

We consider three different noise distributions, $\epsilon$, as follows:
\begin{itemize}
	\item[$(i)$] 
	For Gaussian noise $\epsilon$, the location parameter is set to zero. The scale parameter is adjusted so that the noise’s standard deviation is one-third of that of $f(X)$, in line with \cite{tipping2001sparse}. Specifically, this means $(\mathrm{Var}(f(X)) / \mathrm{Var}(\epsilon))^{1/2} = 3$.
	\item[$(ii)$] 
	For Cauchy noise $\epsilon$, as in Example \ref{ex::cauchy}, the scale parameter $s$ is selected to achieve a signal-to-noise power ratio of $3$. Thus, we choose $s$ so that $( \mathbb{E} |f(X)|^{1/2} / \mathbb{E} |\epsilon|^{1/2} )^2 = 3$.
	\item[$(iii)$] 
	For Pareto noise $\epsilon$, as in Example \ref{ex::pareto}, the shape parameter $\zeta$ is set to $2.01$. The scale parameter $s$ is determined based on a signal-to-noise ratio of $3$, specifically by choosing $s$ such that $( \mathbb{E} |f(X)|^{1/3} / \mathbb{E} |\epsilon|^{1/3} )^3 = 3$.
\end{itemize}

To generate the response variable, noise was added to the regression function, resulting in $Y = f(X) + \epsilon$. For each of Friedman’s functions, we produced $1,000$ noisy observations using the three different noise types described above, which were used for model training and cross-validation. Furthermore, an additional $1,000$ noise-free observations were generated for testing.

For KLAD, MCCR, and KCRR, the hyperparameter grids for the regularization parameter $\lambda$ and the bandwidth $\gamma$ are set to $\{ 0.1, 0.01, \ldots, 10^{-5} \}$ and $\{ 0.5, 0.25, \ldots, 2^{-5} \}$, respectively. The grid for the squared scale parameter $\sigma^2$ in the Cauchy and correntropy losses is selected as $\{10^{-1}, \dots, 10^{-8}\}$. These three models are fit on the standardized data.
For KBHR, the hyperparameter grid for the scale parameter $\sigma$ in the Huber loss, defined as $L_{\mathrm{Huber}}(y,f(x)):= (y - f(x))^2 \eins \{ |y - f(x)| \leq \sigma \} + (\sigma |y - f(x)| - \sigma^2 / 2)$, is set to $\{ 1, 5, 10, 20, 50, 100, 200, 300, 500, 1000 \}$. Prior to fitting KBHR, we standardize the feature variables.

To fit KLAD and KBHR, we employ stochastic gradient descent as implemented in the {\tt Sklearn} package. For MCCR, we use IRLS as recommended by \cite{feng2015learning}.

The hyperparameters for all these algorithms are selected using 10-fold cross-validation, with the \textit{mean absolute error} (\textit{MAE}) as the selection criterion. The MAE is defined as
\begin{align*}
	\mathrm{MAE}(\widehat{f}) := \frac{1}{k} \sum_{i=1}^k |y_i - \widehat{f}(x_i)|,
\end{align*}
where $(x_i, y_i)_{i=1}^k$ represent the $k$ cross-validation samples.

To evaluate the performance of these algorithms on the test data, we use two metrics: the \textit{MAE} defined as $\mathrm{MAE}(\widehat{f}):= \frac{1}{m} \sum_{i=1}^m |f^*(x_i) - \widehat{f}(x_i)|$ and the \textit{relative sum of squared error} (\textit{RSSE}) defined as
\begin{align*}
	\mathrm{RSSE}(\widehat{f}) := \frac{\sum_{i=1}^m (f^*(x_i) - \widehat{f}(x_i))^2}{\sum_{i=1}^m (f^*(x_i) - \bar{f}^*)^2},
\end{align*}
where $(x_i)_{i=1}^m$ are the $m$ test samples and $\bar{f}^*$ is the mean of the $f^*(x_i)$ values for $i \in [m]$. We conduct our experiments ten times and report the average values and standard errors for both metrics in Tables \ref{tab::FriedmanMAE} and \ref{tab::FriedmanRSSE}, respectively.

\begin{table*}[!h]
	\centering
	\captionsetup{justification=centering}
	\vspace{0pt}
	\caption{MAE Performance on Friedman's functions under different noise types.}
	\label{tab::FriedmanMAE}
	\vspace{-6pt}
	\resizebox{0.86\textwidth}{!}{
		\begin{tabular}{cc|cccc}
			\toprule
			Dataset & Noise & KLAD & KBHR & MCCR & KCRR  \\ \midrule
			\multirow{3}{*}{$(I)$} & $(i)$ & 1.4241	± 0.0155 & 1.0441 ± 0.0096 & 0.7424 ± 0.0197&		\textbf{0.6122 ± 0.0182} \\
			& $(ii)$ & 1.4055 ± 0.0481 & 1.3600 ± 0.0340 & 0.9464 ± 0.0251 &	\textbf{0.9026 ± 0.0188} \\
			& $(iii)$ & 1.2748 ± 0.0957 & 1.8464 ± 0.2287 & 0.5498 ± 0.0491 & \textbf{0.3821 ± 0.0312} \\
			\midrule
			\multirow{3}{*}{$(II)$} & $(i)$ & 54.0076 ± 1.1573 & 26.7090 ± 0.5068 & 24.1177 ± 0.7749&   \textbf{22.3529 ± 0.7565} \\
			& $(ii)$ & 39.8103 ± 2.2581 & 34.5415 ± 1.1040 & \textbf{18.6588	± 0.6364} &	19.4297	± 1.3433 \\
			& $(iii)$ & 37.1561 ± 2.4826 & 70.4954 ± 16.2335 & 9.1603 ± 1.0183&	\textbf{6.7785 ± 1.4664} \\
			\midrule
			\multirow{3}{*}{$(III)$} & $(i)$ & 0.0945 ± 0.0023 & 0.1181 ± 0.0028 & 0.0606 ± 0.0009&		\textbf{0.0463 ± 0.0007} \\
			& $(ii)$ & 0.1159 ± 0.0032 & 0.1336 ± 0.0050 & 0.1005 ± 0.0028 &	\textbf{0.0902 ± 0.0031} \\
			& $(iii)$ & 0.1207 ± 0.0099 & 0.1725 ± 0.0059 & 0.0547 ± 0.0017 &	\textbf{0.0425 ± 0.0012} \\
			\bottomrule
		\end{tabular}
	}	
	\begin{tablenotes}
		\footnotesize
		\item[*]\qquad\qquad For each dataset and each noise, we denote the best performance with \textbf{bold}.
	\end{tablenotes}
\end{table*}

\begin{table*}[!h]
	\centering
	\captionsetup{justification=centering}
	\vspace{0pt}
	\caption{RSSE Performance on Friedman's functions under different noise types.}
	\label{tab::FriedmanRSSE}
	\vspace{-6pt}
	\resizebox{0.86\textwidth}{!}{
		\begin{tabular}{cc|cccc}
			\toprule
			Dataset & Noise & KLAD & KBHR & MCCR & KCRR  \\ \midrule
			\multirow{3}{*}{$(I)$} & $(i)$ & 0.1619 ± 0.0041 & 0.0841 ± 0.0019 & 0.0406 ± 0.0021&		\textbf{0.0275 ± 0.0018} \\
			& $(ii)$ & 0.1464 ± 0.0101 & 0.1330 ± 0.0069 & 0.0659 ± 0.0035&	\textbf{0.0603 ± 0.0020}\\
			& $(iii)$ & 0.1308 ± 0.0208 & 0.2665 ± 0.0622 & 0.0235 ± 0.0048 &	\textbf{0.0141 ± 0.0023} \\
			\midrule
			\multirow{3}{*}{$(II)$} & $(i)$ & 0.0465 ± 0.0017 & 0.0094 ± 0.0001 & 0.0068 ± 0.0005&		\textbf{0.0057 ± 0.0004} \\
			& $(ii)$ & 0.0227 ± 0.0028 & 0.0153 ± 0.0011 & \textbf{0.0041 ± 0.0002}&	0.0046 ± 0.0006 \\
			& $(iii)$ & 0.0199 ± 0.0045 & 0.0952 ± 0.0387 & 0.0012 ± 0.0003&	\textbf{0.0011 ± 0.0007} \\
			\midrule
			\multirow{3}{*}{$(III)$} & $(i)$ & 0.3334 ± 0.0078 & 0.3940 ± 0.0046 & 0.0994 ± 0.0060&		\textbf{0.0653 ± 0.0045} \\
			& $(ii)$ & 0.4060 ± 0.0130 & 0.4973 ± 0.0164 & 0.2263 ± 0.0148&	\textbf{0.1735 ± 0.0121} \\
			& $(iii)$ & 0.3236 ± 0.0399 & 0.7692 ± 0.0603 & 0.0929 ± 0.0072& \textbf{0.0792	± 0.0075} \\
			\bottomrule
		\end{tabular}
	}	
	\begin{tablenotes}
		\footnotesize
		\item[*]\qquad\qquad For each dataset and each noise, we denote the best performance with \textbf{bold}.
	\end{tablenotes}
\end{table*}

Tables \ref{tab::FriedmanMAE} and \ref{tab::FriedmanRSSE} show that KCRR consistently outperforms KLAD, KBHR, and MCCR across Friedman’s functions, under various types of noise. Notably, KCRR demonstrates a significant advantage in scenarios with Pareto noise (case $(iii)$), which lacks a finite $1/2$-order moment. This result highlights that the Cauchy loss function offers greater robustness against extremely heavy-tailed noise compared to the other loss functions.

\subsection{Real-world Data Experiments}

\begin{wraptable}{r}{0.5\textwidth}
		\centering
		\captionsetup{justification=centering}
		\vspace{-1pt}
		\caption{\normalsize{Descriptions of Real Data}}
		\label{tab::RealdataDescription}
		\begin{tabular}{l|rr}
			\toprule
			Dataset  & $n$ & $d$    \\
			\midrule
			{\tt Computer} & $209$ & $10$  \\
			{\tt Facebook} & $500$ & $17$  \\
			{\tt Housing} & $506$ & $13$  \\
			{\tt Yacht} & $308$ & $7$ \\
			\bottomrule    
		\end{tabular}
\end{wraptable}

We evaluate the performance of our models using four real-world regression datasets from the UCI Machine Learning Repository \citep{kelly2007uci}: {\tt Computer Hardware}, {\tt Facebook Metrics}, {\tt Boston Housing}, and {\tt Yacht Hydrodynamics}. The details of these datasets, including the sample size $n$ and the number of features $d$, are provided in Table~\ref{tab::RealdataDescription}.

For all the robust methods, the bandwidth parameter $\gamma$ in the kernel function is chosen from the grid $\{0.5, 0.25, \ldots, 2^{-6}\}$, and the squared scale parameter $\sigma^2$ in the Cauchy loss and correntropy loss is selected from $\{10^{-3}, 10^{-2}, \ldots, 10\}$. The remaining parameter grids are kept the same as in the synthetic experiments. Each dataset is randomly split into $70\%$ for training and $30\%$ for testing. The experiments are repeated ten times, with parameters selected using $10$-fold cross-validation based on the MAE criterion.

\begin{table*}[!h]
	\centering
	\captionsetup{justification=centering}
	\vspace{0pt}
	\caption{MAE Performance on real-world datasets.}
	\label{tab::realdataMAE}
	\vspace{-6pt}
	\resizebox{0.86\textwidth}{!}{
		\begin{tabular}{c|cccc}
			\toprule
			Dataset & KLAD & KBHR & MCCR &  KCRR \\ \midrule
			{\tt Computer} & 44.4389 ± 4.5196 & 36.1817 ± 2.8616 & 30.3279 ± 2.7860 & \textbf{28.3316 ± 2.1660} 
			\\
			{\tt Facebook} & 79.2629 ± 6.7764 & 51.3577 ± 4.2343 & 13.5840 ± 2.0781 & \textbf{11.5963 ± 2.0316} 
			\\
			{\tt Housing} & 3.1702 ± 0.0796 & 2.5449 ± 0.0687 & 2.1804 ± 0.0411& \textbf{2.0714 ± 0.0422}
			\\
			{\tt Yacht} & 6.3457 ± 0.4402 & 5.2080 ± 0.2433 & 1.0794 ± 0.0650 & \textbf{0.3984 ± 0.0315}
			\\\bottomrule
		\end{tabular}
	}
	\begin{tablenotes}
		\footnotesize
		\item[*]\qquad\qquad\qquad\qquad For each dataset, we denote the best performance with \textbf{bold}.
	\end{tablenotes}
\end{table*}

\begin{table*}[!h]
	\centering
	\captionsetup{justification=centering}
	\vspace{0pt}
	\caption{RSSE Performance on real-world datasets.}
	\label{tab::realdataRSSE}
	\vspace{-6pt}
	\resizebox{0.86\textwidth}{!}{
		\begin{tabular}{c|cccc}
			\toprule
			Dataset & KLAD & KBHR & MCCR &  KCRR \\ \hline
			{\tt Computer} & 0.4924 ± 0.0545 & 0.2676 ± 0.0304 & 0.2635	± 0.0898& \textbf{0.1546 ± 0.0268} 
			\\
			{\tt Facebook} & 0.5594 ± 0.0370 & 0.1620 ± 0.0270 & 0.0643	± 0.0225& \textbf{0.0614 ± 0.0230} 
			\\
			{\tt Housing} & 0.3318 ± 0.0132 & 0.1948 ± 0.0144 & 0.1443 ± 0.0107& \textbf{0.1201 ± 0.0058} 
			\\
			{\tt Yacht} & 0.6305 ± 0.0329 & 0.2585 ± 0.0132 & 0.0119 ± 0.0018&	\textbf{0.0026 ± 0.0007} 
			\\\bottomrule
		\end{tabular}
	}	
	\begin{tablenotes}
		\footnotesize
		\item[*]\qquad\qquad\qquad\qquad For each dataset, we denote the best performance with \textbf{bold}.
	\end{tablenotes}
\end{table*}

Tables \ref{tab::realdataMAE} and \ref{tab::realdataRSSE} show that KCRR consistently outperforms other kernel-based robust regression methods on real-world datasets, as measured by both the MAE and RSSE metrics. These results underscore the effectiveness and adaptability of the Cauchy loss in managing various types of noise commonly found in real-world data.

\section{Proofs}\label{sec::proofs}

In this section, we present the proofs for the results in previous sections. Specifically, Section \ref{subsec::proofrobust} demonstrates the finiteness of the Cauchy risk and establishes the corresponding Bayes function under the generalized Cauchy noise assumption. Section \ref{subsec::proofmain} provides detailed proofs for the main theoretical results outlined in Section \ref{sec::mainresults}. Lastly, Section \ref{subsec::prooferror} covers the proof related to the error analysis for the $L_2$-risk of KCRR in Section \ref{sec::ErrorAnalysis}.

\subsection{Proofs Related to Section \ref{sec::robustreg}}\label{subsec::proofrobust}

\begin{proof}[of Lemma \ref{lem::stronger}]
	For any $x \geq \sqrt{2}$, we have $\log(1+x^2) \leq \log(2x^2) \leq \log(x^4) = 4\log x$. First, we prove that $4\log x \leq x^p$ for any $p \leq 1/4$ and $x\geq e^{(1/p)^{4/p}}$.  To this end, we construct the function $g(t):=t^p-4\log t$ for any $t>0$. The derivative function of $g$ is given by $g'(t)=pt^{p-1}-4/t=t^{-1}(pt^p-4)$, which is larger than zero for $t>(4/p)^{1/p}$. Therefore, the function $g$ is increasing on $t>(4/p)^{1/p}$. Now we show $(4/p)^{1/p} \leq e^{(1/p)^{4/p}}$, which is equivalent to $(1/p) \log(4/p) \leq (1/p)^{4/p}$.
	Since $1+x \leq e^x$ for $x>0$, we have $ 1+\log(4/p) \leq 4/p$ and thus we get
	\begin{align*}
		\log(4/p) \leq 4/p-1 \leq {(1/p)}^{4/p-1},
	\end{align*}
	where the last inequality follows from $x \leq e^x$. Therefore, we get $(1/p)\log(4/p) \leq {(1/p)}^{4/p}$ and thus $g$ is increasing on $t > e^{(1/p)^{4/p}}$. 
	In addition, we check that $g(e^{(1/p)^{4/p}}) >0$. Specifically, we have 
	\begin{align*}
		g(e^{(1/p)^{4/p}}) = e^{(1/p)^{4/p-1}} - e^{\log 4 + (4/p)\log(1/p)}.
	\end{align*}
	Since $1/p \geq 4$ and $1/p > \log(1/p) $ for any $p\leq 1/4$, we have 
	\begin{align*}
		(1/p)^{4/p-1} &> (1/p)^{3/p} = (1/p)^{3/p-3}(1/p)^{3} > 2(1/p)^{3} > (1/p^2 + 2) (1/p) 
		\\
		& > 1/p + (1/p^2 + 1) (1/p) > 1/p + (4/p) \log (1/p) > \log 4 + (4/p) \log (1/p).
	\end{align*}
	Therefore, we get $g(e^{(1/p)^{4/p}}) > 0$ and thus $g(t) > 0$ for any $t \geq e^{(1/p)^{4/p}}$. Thus, we finish the proof of $\log(1+x^2) \leq x^p$ for any $x \geq e^{(1/p)^{4/p}}$ and $p \leq 1/4$.
	
	In the following, we show that $\log(1+x^2) \leq x^p$ for any $x \geq 2^{64}$ and $p > 1/4$. Under $p > 1/4$, we have $x^p \geq x^{1/4}$ and thus it suffices to show that $4\log x < x^{1/4}$ for $x \geq 2^{16}$. To this end, we construct the function $h(t):= t^{1/4}-4\log t$ on $t>0$. The derivative function of $h$ is given by $h'(t)= (1/4)t^{-3/4}-4/t = (4t)^{-1}(t^{1/4}-16)$, which is larger than zero for any $t > 2^{16}$. Therefore, $h(t)$ is increasing on $t > 2^{16}$. Moreover, $h(2^{64})=2^{16}-4 \times 64\log 2 > 0$. Therefore $h(t) > 0$ for any $t > 2^{16}$.
	
	By combining these two sides, we get 
	$\log(1+x^2) \leq x^p$ for any $p>0$ and $x \geq 2^{64} \vee e^{(1/p)^{4/p}}$. Then we have 
	\begin{align*}
		\mathbb{E}\log(1 + |\epsilon|^2) &= \mathbb{E}(\log(1 + |\epsilon|^2) \eins\{|\epsilon|\leq 2^{64} \vee e^{(1/p)^{4/p}}\}) + \mathbb{E}(\log(1 + |\epsilon|^2) \eins\{|\epsilon| >  2^{64} \vee e^{(1/p)^{4/p}}\})
		\nonumber\\
		&\leq \log(1 + 2^{128} \vee e^{2(1/p)^{4/p}}) + \mathbb{E}(|\epsilon|^p \eins\{|\epsilon| >  2^{64} \vee e^{(1/p)^{4/p}}\})
		\nonumber\\
		&\leq \log(1 + 2^{128} \vee e^{2(1/p)^{4/p}}) + \mathbb{E}(|\epsilon|^p).
	\end{align*}
	Therefore if $\mathbb{E}(|\epsilon|^p)<\infty$, we get $\mathbb{E}\log(1 + |\epsilon|^2) < \infty$. 
\end{proof}

\begin{proof}[of Lemma \ref{lem::finiterisk}]
	Using the definition of the Cauchy loss $L$ and the inequalities $(a+b)^2 \leq 2(a^2 + b^2)$, we get 
	\begin{align}\label{eq::lossbound}
		& L(y,f(x)) 
		= \sigma^2 \log \biggl( 1 + \frac{(y - f(x))^2}{\sigma^2} \biggr) 
		\nonumber\\
		& \leq \sigma^2 \log \biggl( 1 + \frac{2(f(x) - f^*(x))^2}{\sigma^2} + \frac{2(f^*(x)-y)^2}{\sigma^2} \biggr)
		\nonumber\\
		& \leq \sigma^2 \log \biggl( 1 + \frac{4 (\|f\|_{\infty}^2 + \|f^*\|_{\infty}^2)}{\sigma^2} + \frac{2(f^*(x) - y)^2}{\sigma^2} \biggr)
		\nonumber\\
		&= \sigma^2 \log \biggl( 1 + \frac{4 (\|f\|_{\infty}^2 + \|f^*\|_{\infty}^2)}{\sigma^2} \biggr) + \sigma^2 \log \biggl( 1 + \frac{2(f^*(x) - y)^2}{\sigma^2 + 4(\|f\|_{\infty}^2 + \|f^*\|_{\infty}^2)} \biggr).
	\end{align}
	If $\sigma^2 + 4(\|f\|_{\infty}^2 + \|f^*\|_{\infty}^2) \geq 2$, then we have 
	\begin{align}\label{eq::lossbound1}
		\mathbb{E} L(Y,f(X)) 
		& \leq \sigma^2 \log \biggl( 1 + \frac{4 (\|f\|_{\infty}^2 + \|f^*\|_{\infty}^2)}{\sigma^2} \biggr) + \sigma^2 \mathbb{E} \log \bigl( 1 + (f^*(X) - Y)^2 \bigr)
		\nonumber\\
		& = \sigma^2 \log \biggl( 1 + \frac{4 (\|f\|_{\infty}^2 + \|f^*\|_{\infty}^2)}{\sigma^2} \biggr) + \sigma^2 \mathbb{E}_{\epsilon} \log \bigl( 1 + \epsilon^2 \bigr) 
		< \infty,
	\end{align}
	where the last inequality follows from Assumption \ref{ass::logmoment}. 
	
	Otherwise if $\sigma^2 + 4(\|f\|_{\infty}^2+ \|f^*\|_{\infty}^2) < 2$,
	by using \eqref{eq::lossbound} and $a \log(1+t/a) \leq \log(1+t)$ for any $t \geq 0$ and $a \in (0, 1]$, we have for any $(x, y) \in \mathcal{X}\times \mathcal{Y}$, 
	\begin{align*}
		L(y,f(x)) 
		& \leq \sigma^2 \log \biggl( 1 + \frac{4 (\|f\|_{\infty}^2 + \|f^*\|_{\infty}^2)}{\sigma^2} \biggr) 
		+ \frac{2\sigma^2 \log \bigl( 1 + (f^*(x) - y) \bigr)^2}{\sigma^2 + 4 (\|f\|_{\infty}^2 + \|f^*\|_{\infty}^2)}
		\\
		& \leq \sigma^2 \log \biggl( 1 + \frac{4 (\|f\|_{\infty}^2 + \|f^*\|_{\infty}^2)}{\sigma^2} \biggr) 
		+ 2\log \bigl( 1 + (f^*(x) - y) \bigr)^2.
	\end{align*}
	Therefore, by Assumption \ref{ass::logmoment}, we have
	\begin{align*}
		\mathbb{E} L(Y,f(X)) 
		\leq \sigma^2 \log \biggl( 1 + \frac{4 (\|f\|_{\infty}^2 + \|f^*\|_{\infty}^2)}{\sigma^2} \biggr) 
		+ 2\mathbb{E}_{\epsilon} \log \bigl( 1 + \epsilon^2 \bigr) 
		< \infty.
	\end{align*}
	This together with \eqref{eq::lossbound1} yields the conclusion.
\end{proof}

\begin{proof}[of Lemma \ref{lem::optimality}]
	By the definition of $\mathcal{R}_{L,P}$, for any function $f : \mathcal{X} \to \mathbb{R}$, we have
	\begin{align*}
		\mathcal{R}_{L, P}(f) = \mathbb{E}_X \mathbb{E}_{Y|X} L(Y, f(X)).
	\end{align*}
	Let the inner risk of the Cauchy loss be denoted as
	\begin{align}\label{eq::Rxf}
		\mathcal{R}_x(f) &:= \mathbb{E}_{Y|X=x} L(Y, f(X)) 
		= \int_{\mathcal{Y}} \sigma^2 \log \left( 1 + \frac{(y - f(x))^2}{\sigma^2} \right) p(y|x) \, dy.
		\nonumber\\
		&= \int_{\mathbb{R}} \sigma^2 \log \left( 1 + \frac{(\epsilon + f^*(x) - f(x))^2}{\sigma^2} \right) p(\epsilon) \, d\epsilon.
	\end{align}
	Let us define a function of $u$ as  
	\begin{align}\label{eq::gu}
		g(u) := \int_{\mathbb{R}} \sigma^2 \log \left( 1 + \frac{(\epsilon + u)^2}{\sigma^2} \right) p(\epsilon) \, d\epsilon.
	\end{align}
	By taking the derivative of $g(u)$ with respect to $u$, 
	we obtain
	\begin{align*}
		g'(u) := \int_{\mathbb{R}} \sigma^2 \left( \frac{2(\epsilon + u)}{\sigma^2+(\epsilon + u)^2}\right) p(\epsilon) \, d\epsilon
	\end{align*}
	By Assumption \ref{ass::decreasingtails}, for any $u > 0$ and any $a > 0$, we have
	\begin{align*}
		p(\epsilon = a - u) > p(\epsilon = -a - u),
	\end{align*}
	This implies that $g'(u) > 0$ for any $u > 0$. Conversely, we can prove that $g'(u) < 0$ for any $u < 0$. Moreover, it is clear that $g'(u) = 0$ when $u = 0$ by using Assumption \ref{ass::symmetry}. As a result, we can conclude that the inner risk $g(u)$ behaves in the following manner:
	\begin{align}\label{eq::monotonicgu}
		g(u) 
		\begin{cases}
			\text{is decreasing}, & \text{ if } u < 0;
			\\
			\displaystyle = \sigma^2 \int_{\mathbb{R}} \log (1 + \epsilon^2 / \sigma^2) p(\epsilon) \, d \epsilon, & \text { if } u = 0;
			\\
			\text{is increasing}, & \text{ if } u > 0.
		\end{cases}
	\end{align}
	Therefore, for any $u \in \mathbb{R}$, there holds $g(u) \geq g(0) = \sigma^2 \int_{\mathbb{R}} \log (1 + \epsilon^2 / \sigma^2) p(\epsilon) \, d \epsilon$ and $u=0$ is the unique minimal point of $g(u)$. Thus, for any $f$ and any $x$, there holds $g(f^*(x) - f(x)) \geq g(0)$. By the definition of $g$, we have $g(f^*(x) - f(x)) = \mathcal{R}_x(f)$ and $g(0)=\mathcal{R}_x(f^*)$. Therefore, $\mathcal{R}_x(f) \geq \mathcal{R}_x(f^*)$ and the equation holds if and only if $f(x) = f^*(x)$. This leads to the result that the minimal inner risk $\mathcal{R}_x(f)$ is achieved when $f(x) = f^*(x)$, i.e.,
	\begin{align*}
		\inf \{ \mathcal{R}_x(f) \mid f: \mathcal{X} \to \mathcal{Y} \text{ measurable} \}
		= \mathcal{R}_x(f^*).
	\end{align*}
	By taking the expectation with respect to $P_X$, we extend this result to the overall risk, yielding that
	\begin{align*}
		\inf \{\mathcal{R}_{L, P}(f) \mid f: \mathcal{X} \to \mathcal{Y} \text{ measurable} \}
		= \mathcal{R}_{L,P}(f^*).
	\end{align*}
	This demonstrates that the true regression function minimizes the Cauchy risk and completes the proof.
\end{proof}

\begin{proof}[of Lemma \ref{lem::validateclip}]
	Since $M \geq \|f^*\|_{\infty}$, we can analyze the relationship between $\wideparen{f}(x)$ and $f^*(x)$ based on the value of $f(x)$.
	\begin{description}
		\item[Case 1:] $f(x) \geq f^*(x)$. 
		In this scenario, we have $f^*(x) \leq \wideparen{f}(x) \leq f(x)$. By the monotonicity property of $g(u)$ stated in \eqref{eq::monotonicgu}, this implies $g(f^*(x) - \wideparen{f}(x)) \leq g(f^*(x) - f(x))$. By \eqref{eq::Rxf} and \eqref{eq::gu}, we have $\mathcal{R}_x(f) = g(f^*(x) - f(x))$. Thus we obtain $\mathcal{R}_x(\wideparen{f}) \leq \mathcal{R}_x(f)$.
		\item[Case 2:] $f(x) \leq f^*(x)$.  
		Here, it follows that $f(x) \leq \wideparen{f}(x) \leq f^*(x)$. Again, using the monotonicity property in \eqref{eq::monotonicgu}, this implies $g(f^*(x) - \wideparen{f}(x)) \leq g(f^*(x) - f(x))$.
		By \eqref{eq::Rxf} and \eqref{eq::gu}, we have $\mathcal{R}_x(f) = g(f^*(x) - f(x))$. Thus we find $\mathcal{R}_x(\wideparen{f}) \leq \mathcal{R}_x(f)$.
	\end{description}
	Combining both cases, we conclude that $\mathcal{R}_x(\wideparen{f}) \leq \mathcal{R}_x(f)$ holds for any $x$. By taking the expectation with respect to $P_X$, we extend this result to the overall risk, which establishes the desired assertion.
\end{proof}

\subsection{Proofs Related to Section \ref{sec::mainresults}}\label{subsec::proofmain}

\subsubsection{Proofs Related to Section \ref{sec::CalibrationInequalities}}

\begin{lemma}\label{lem::relation}
	Let Assumptions \ref{ass::logmoment}, \ref{ass::symmetry}, and \ref{ass::decreasingtails} hold. Additionally, let $L$ be the Cauchy loss function, with $f^*$ representing the true regression function. Then for any regressor $f: \mathcal{X} \to \mathbb{R}$, the following holds
	\begin{align*}
		\mathcal{R}_{L,P}(f) - \mathcal{R}_{L,P}^* 
		\leq \mathbb{E}_{P} (f(X) - f^*(X))^2.
	\end{align*}
\end{lemma}

\begin{proof}[of Lemma \ref{lem::relation}]
	By Lemma \ref{lem::optimality}, we have
	\begin{align}\label{eq::reexp}
		\mathcal{R}_{L,P}(f) - \mathcal{R}_{L,P}^* 
		= \mathbb{E}_X \mathbb{E}_{Y|X} \bigl( L(Y,f(X)) - L(Y,f^*(X)) \bigr).
	\end{align}
	For the Cauchy loss function, we have
	\begin{align*}
		& \mathbb{E}_{Y|X=x} \bigl( L(Y,f(X)) - L(Y,f^*(X) \bigr) 
		\\
		& = \int_{\mathcal{Y}} \sigma^2 \biggl( \log \biggl( 1 + \frac{(y - f(x))^2}{\sigma^2} \biggr) - \log \biggl( 1 + \frac{(y - f^*(x))^2}{\sigma^2} \biggr) \biggr) p(y|x) \, dy
		\\
		& = \int_{\mathcal{Y}} \sigma^2 \log \biggl( \frac{(y - f(x))^2 + \sigma^2}{(y - f^*(x))^2 + \sigma^2} \biggr) p(y|x) \, dy
		\\
		& = \int_{\mathcal{Y}} \sigma^2 \log \biggl( 1 + \frac{(y - f(x))^2 - (y - f^*(x))^2}{(y - f^*(x))^2 + \sigma^2} \biggr) p(y|x) \, dy.
	\end{align*}
	Using the inequality $\log(1 + x) \leq x$ for $x\in (-1, \infty)$ and substituting the regression model $\epsilon = Y - f^*(X)$ into the expression, we get
	\begin{align*}
		& \mathbb{E}_{Y|X=x} \bigl( L(Y,f(X)) - L(Y,f^*(X) \bigr) 
		\\
		& \leq \sigma^2 \int_{\mathcal{Y}} \frac{(y - f(x))^2 - (y - f^*(x))^2}{(y - f^*(x))^2 + \sigma^2} \cdot p(y|x) \, dy
		\\
		& = \sigma^2 \int_{\mathcal{Y}} \frac{\bigl( (y - f^*(x)) + (f^*(x) - f(x)) \bigr)^2 - (y - f^*(x))^2}{(y - f^*(x))^2 + \sigma^2} \cdot p(y|x) \, dy
		\\
		& = \sigma^2 \int_{\mathcal{Y}} \frac{2 (y - f^*(x)) (f^*(x) - f(x)) + (f^*(x) - f(x))^2}{(y - f^*(x))^2 + \sigma^2} \cdot p(y|x) \, dy
		\\
		& = \sigma^2 \int_{\mathcal{Y}} \frac{2 \epsilon (f^*(x) - f(x)) + (f^*(x) - f(x))^2}{\epsilon^2 + \sigma^2} \cdot p(\epsilon|x) \, dy.
	\end{align*}
	According to the symmetry assumption stated in Assumption \ref{ass::symmetry}, for any $x$, we have $\mathbb{E}(\epsilon/(\epsilon^2 + \sigma^2) | X = x) = 0$. From this, we get
	\begin{align*}
		\mathbb{E}_{Y|X=x} \bigl( L(Y,f(X)) - L(Y,f^*(X) \bigr) 
		&\leq \sigma^2 \int_{\mathcal{Y}} \frac{(f^*(x) - f(x))^2}{\epsilon^2 + \sigma^2} \cdot p(\epsilon) \, d\epsilon
		\\
		& \leq \int_{\mathcal{Y}} (f^*(x) - f(x))^2 p(\epsilon) \, d\epsilon 
		\\
		& = (f^*(x) - f(x))^2.
	\end{align*}
	This together with \eqref{eq::reexp} yields the assertion. 
\end{proof}

\begin{proof}[of Theorem \ref{lem::rell2Cauchylargesigma}]
	By the definition of the Cauchy loss $L$ and $\mathcal{R}_{L,P}^* = \mathbb{E} L(Y, f^*(X))$, we have
	\begin{align*}
		\mathcal{R}_{L,P}(\wideparen{f}) - \mathcal{R}_{L,P}^* 
		& = \mathbb{E}_X \mathbb{E}_{Y|X} \bigl( L(Y, \wideparen{f}(X)) - L(Y, f^*(X)) \bigr) 
		\\
		& = \mathbb{E}_X \mathbb{E}_{Y|X} \biggl( \sigma^2 \log \biggl( 1 + \frac{(Y - \wideparen{f}(X))^2}{\sigma^2} \biggr) - \sigma^2 \log \biggl( 1 + \frac{(Y - f^*(X))^2}{\sigma^2} \biggr) \biggr)
		\\
		& = \mathbb{E}_X \mathbb{E}_{Y|X} \sigma^2 \log \biggl( \frac{(Y - \wideparen{f}(X))^2 + \sigma^2}{(Y - f^*(X))^2 + \sigma^2} \biggr)
		\\
		& = \mathbb{E}_X \mathbb{E}_{Y|X} \sigma^2 \log \biggl( 1 + \frac{(Y - \wideparen{f}(X))^2 - (Y - f^*(X))^2}{(Y - f^*(X))^2 + \sigma^2} \biggr).
	\end{align*}
	Let us define $\xi(X) := \wideparen{f}(X) - f^*(X)$.
	Since $\epsilon = Y-f^*(X)$, we have
	\begin{align*}
		\mathcal{R}_{L,P}(\wideparen{f}) - \mathcal{R}_{L,P}^* 
		= \mathbb{E}_X \mathbb{E}_{\epsilon} \sigma^2 \log \biggl( 1 + \frac{(\epsilon - \xi(X))^2 - \epsilon^2}{\epsilon^2 + \sigma^2} \biggr).
	\end{align*}
	Obviously, we have $\|\xi\|_{\infty} \leq \|f^*\|_{\infty} + \|\wideparen{f}\|_{\infty} \leq 2 M$.
	Thus if $|\epsilon| \leq M$, then 
	\begin{align*}
		|(\epsilon - \xi(X))^2 - \epsilon^2|
		\leq 2 |\epsilon| |\xi(X)| + \xi^2(X) 
		\leq 4 M^2 + 4 M^2 
		= 8M^2. 
	\end{align*}
	Since $\sigma \geq 4M$, we then have
	\begin{align*}
		\frac{|(\epsilon - \xi(X))^2 - \epsilon^2|}{\epsilon^2 + \sigma^2} 
		\leq \frac{8M^2}{\sigma^2} 
		\leq  \frac{8M^2}{16 M^2} 
		\leq \frac{2}{3}.
	\end{align*}
	Otherwise if $|\epsilon| > M$, using the inequality $2 a b \leq a^2 + b^2$, we get 
	\begin{align*}
		\frac{|(\epsilon - \xi(X))^2 - \epsilon^2|}{\epsilon^2 + \sigma^2} 
		& \leq \frac{2|\epsilon| \xi(X) + \xi^2(X)}{\epsilon^2 + \sigma^2} 
		\leq \frac{4M |\epsilon| + 4M^2}{\epsilon^2 + 16M^2} 
		\\
		& = \frac{(2/3) \cdot 6M |\epsilon| + 4M^2}{\epsilon^2 + 16M^2}
		\leq \frac{(2/3) \cdot (\epsilon^2 + 9M^2) + 4M^2}{\epsilon^2 + 16M^2} 
		\leq \frac{2}{3}.
	\end{align*}
	Consequently, we obtain that when $\sigma \geq 4M$,
	for any $\epsilon$ and $X$, 
	\begin{align*}
		\frac{|(\epsilon - \xi(X))^2 - \epsilon^2|}{\epsilon^2 + \sigma^2} 
		\leq \frac{2}{3}.
	\end{align*}
	
	Using the inequality $\log(1+t) \geq t - 3t^2/2$ for $t \in [-2/3, \infty)$, and the symmetry assumption stated in Assumption \ref{ass::symmetry}, we get
	\begin{align}\label{eq::key}
		\mathcal{R}_{L,P}(\wideparen{f}) - \mathcal{R}_{L,P}^* 
		& \geq \sigma^2 \mathbb{E}_X \mathbb{E}_{\epsilon} \biggl( \frac{(\epsilon - \xi(X))^2 - \epsilon^2}{\epsilon^2 + \sigma^2} - \frac{3((\epsilon - \xi(X))^2 - \epsilon^2)^2}{2 (\epsilon^2 + \sigma^2)^2} \biggr)
		\nonumber\\
		& = \sigma^2 \mathbb{E}_X \mathbb{E}_{\epsilon} \biggl( \frac{\xi(X)^2}{\epsilon^2 + \sigma^2} - \frac{3\xi(X)^4 + 12 \epsilon^2 \xi(X)^2}{2 (\epsilon^2 + \sigma^2)^2} \biggr)
		\nonumber\\
		& = \mathbb{E}_X \xi(X)^2 - \mathbb{E}_X \mathbb{E}_{\epsilon} \biggl(\frac{\epsilon^2 \xi(X)^2}{\epsilon^2 + \sigma^2} + \frac{3\sigma^2 \xi(X)^4 + 12 \sigma^2 \epsilon^2 \xi(X)^2}{2 (\epsilon^2 + \sigma^2)^2} \biggr)
		\nonumber\\
		& = \mathbb{E}_X \xi(X)^2 - \mathbb{E}_X \biggl(\xi(X)^2 \mathbb{E}_{\epsilon} \biggl( \frac{7\epsilon^2\sigma^2 + \epsilon^4 + 3\xi(X)^2\sigma^2/2}{(\epsilon^2 + \sigma^2)^2}\biggr) \biggr).
	\end{align}
	Since $7\sigma^2 \epsilon^2 / (\sigma^2 + \epsilon^2)^2$ 
	and $\epsilon^4/(\sigma^2 + \epsilon^2)^2$ are bounded, by using the dominated convergence theorem, we have
	\begin{align*}
		\lim_{\sigma \to \infty} \mathbb{E}_{\epsilon} \biggl( \frac{7\epsilon^2\sigma^2}{(\sigma^2 + \epsilon^2)^2} \biggr) 
		& =  \mathbb{E}_{\epsilon} \biggl( \lim_{\sigma \to \infty}  \frac{7\epsilon^2\sigma^2}{(\sigma^2 + \epsilon^2)^2} \biggr) = 0,
		\\
		\lim_{\sigma \to +\infty} \mathbb{E}_{\epsilon} \biggl( \frac{\epsilon^4}{(\sigma^2 + \epsilon^2)^2} \biggr) 
		& = \mathbb{E}_{\epsilon} \biggl( \lim_{\sigma \to +\infty} \frac{\epsilon^4}{(\sigma^2 + \epsilon^2)^2} \biggr) 
		= 0.
	\end{align*}
	Therefore, there exists a large number $c_1$ such that for any $\sigma \geq c_1$, 
	\begin{align}\label{eq::secondfrac}
		\mathbb{E}_{\epsilon} \biggl( \frac{7\epsilon^2\sigma^2}{(\sigma^2 + \epsilon^2)^2} \biggr)  \leq \frac{1}{4}, 
		\qquad 
		\mathbb{E}_{\epsilon} \biggl( \frac{\epsilon^4}{(\sigma^2 + \epsilon^2)^2} \biggr) \leq \frac{1}{4}.
	\end{align}
	Combining \eqref{eq::secondfrac} with \eqref{eq::key} and using $\sigma^4/(\epsilon^2 + \sigma^2)^2 \leq 1$, we obtain
	\begin{align*}
		\mathcal{R}_{L,P}(\wideparen{f}) - \mathcal{R}_{L,P}^* 
		& \geq \frac{1}{2} \cdot \mathbb{E}_X (\wideparen{f}(X) - f^*(X))^2 - \mathbb{E}_X (\wideparen{f}(X) - f^*(X))^4 \cdot \mathbb{E}_{\epsilon} \biggl( \frac{3\sigma^2/2}{(\epsilon^2 + \sigma^2)^2} \biggr)
		\nonumber\\
		&\geq \frac{1}{2} \cdot \mathbb{E}_X (\wideparen{f}(X) - f^*(X))^2 - \mathbb{E}_X (\wideparen{f}(X) - f^*(X))^4 \cdot \frac{3}{2\sigma^2}. 
	\end{align*}
	This is equivalent to
	\begin{align}\label{eq::relationanysigma}
		\|\wideparen{f} - f^*\|^2_{L_2(P_X)} 
		\leq 2\bigl( \mathcal{R}_{L,P}(\wideparen{f}) - \mathcal{R}_{L,P}^*\bigr) + 3\|\wideparen{f} - f^*\|_{L_4(P_X)}^4 /\sigma^2.
	\end{align}
	Given that $\|f^*\|_{\infty} \leq M$ and $\|\wideparen{f}\|_{\infty} \leq M$, we can conclude that
	\begin{align*}
		\|\wideparen{f} - f^*\|_{\infty} \leq \|f^*\|_{\infty} + \|\wideparen{f}\|_{\infty} \leq 2M.
	\end{align*}
	Since $\sigma \geq 4M$, we have
	\begin{align} \label{eq::EstimateAdd}
		\frac{1}{\sigma^2} \cdot \|\wideparen{f} - f^*\|^4_{L_4(P_X)}
		& \leq \frac{1}{\sigma^2} \cdot \|\wideparen{f} - f^*\|^2_{\infty} \cdot \|\wideparen{f} - f^*\|^2_{L_2(P_X)} 
		\nonumber\\
		& \leq \frac{(2M)^2}{(4M)^2} \cdot   \|\wideparen{f} - f^*\|^2_{L_2(P_X)} 
		\leq \frac{1}{4} \cdot \|\wideparen{f} - f^*\|^2_{L_2(P_X)}.
	\end{align}
	Combining \eqref{eq::EstimateAdd} with \eqref{eq::relationanysigma}, we find 
	\begin{align*}
		\|\wideparen{f} - f^*\|^2_{L_2(P_X)} 
		\leq 2\bigl( \mathcal{R}_{L,P}(\wideparen{f}) - \mathcal{R}_{L,P}^*\bigr) + (3/4) \cdot \|\wideparen{f} - f^*\|_{L_4(P_X)}^2,
	\end{align*}
	which is equivalent to
	\begin{align*}
		\|\wideparen{f} - f^*\|^2_{L_2(P_X)}
		\leq 8 \bigl( \mathcal{R}_{L,P}(\wideparen{f}) - \mathcal{R}_{L,P}^* \bigr).
	\end{align*}
	This together with Lemma \ref{lem::relation} yields the assertion.
\end{proof}

\subsubsection{Proofs Related to Section \ref{sec::RatesCauchy}}

\begin{proof}[of Theorem \ref{thm::rates}]
	Propositions \ref{prop::oracle} and \ref{prop::approx} yield that there exist a constant $c_1 > 0$ such that for any $q \in (0,1)$, $\sigma > 4M \vee c_1$, and $\tau > 0$, 
	\begin{align*}
		\lambda \|f_D\|^2_H + \mathcal{R}_{L,P}(\wideparen{f}_D) - \mathcal{R}_{L, P}^* 
		\lesssim \lambda \gamma^{-d} + \gamma^{2\alpha} 
		+ \sigma^2 \left( \tau n^{-1} + \lambda^{-q} \gamma^{-d} n^{-1} \right)
	\end{align*}
	holds with probability at least $1 - 2 e^{- \tau}$. By taking $\tau := \log(2n)$, we obtain
	\begin{align*}
		\lambda \|f_D\|^2_H + \mathcal{R}_{L,P}(\wideparen{f}_D) - \mathcal{R}_{L, P}^* &\lesssim \lambda \gamma^{-d} + \gamma^{2\alpha} 
		+ \sigma^2 \left( \log(2n) \cdot n^{-1} + \lambda^{-q} \gamma^{-d} n^{-1} \right)
		\nonumber\\
		&\lesssim \lambda \gamma^{-d} + \gamma^{2\alpha} 
		+ \sigma^2 \left( n^{-1+q} + \lambda^{-q} \gamma^{-d} n^{-1} \right)
	\end{align*}
	with probability at least $1 - 1/n$. This finishes the proof.  
\end{proof}

\subsubsection{Proofs Related to Section \ref{sec::RatesL2}}

\begin{proof}[of Theorem \ref{thm::ratesanyM}]
	By applying Theorem \ref{lem::rell2Cauchylargesigma} and Theorem \ref{thm::rates}, there exists a constant $c_1$ such that if $\sigma \geq 4n^p \vee c_1$ and $M = n^p\geq \|f^*\|_{\infty}$, then
	\begin{align*}
		\|\wideparen{f}_D - f^*\|^2_{L_2(P_X)} 
		\lesssim \lambda \gamma^{-d} + \gamma^{2\alpha} + \sigma^2 \left( n^{-1+q} + \lambda^{-q} \gamma^{-d} n^{-1} \right)
	\end{align*}
	with probability at least $1 - 1/n$. 
	By choosing 
	\begin{align*}
		\sigma \asymp n^p,
		\qquad 
		\lambda \asymp n^{-\frac{1-2p}{1+q}}, 
		\qquad 
		\gamma \asymp n^{-\frac{1-2p}{(2\alpha+d)(1+q)}}, 
	\end{align*}
	we find that for any $n \geq (\|f\|_{\infty} \vee c_1)^{1/p}$,
	\begin{align*}
		\|\wideparen{f}_D - f^*\|^2_{L_2(P_X)}  \lesssim n^{-\frac{2\alpha(1-2p)}{(2\alpha+d)(1+q)}}.
	\end{align*}
	This concludes the proof.
\end{proof}

\begin{proof}[of Theorem \ref{thm::L2rates}]
	By applying Theorems \ref{lem::rell2Cauchylargesigma} and \ref{thm::rates}, there exists a constant $c_1$ such that if $\sigma \geq 4M_0 \vee c_1$,
	\begin{align*}
		\|\wideparen{f}_D - f^*\|^2_{L_2(P_X)} 
		\lesssim \lambda \gamma^{-d} + \gamma^{2\alpha} + \sigma^2 \left( n^{-1+q} + \lambda^{-q} \gamma^{-d} n^{-1} \right)
	\end{align*}
	with probability at least $1-1/n$. By choosing 
	\begin{align*}
		\lambda \asymp n^{-\frac{1}{1+q}}, 
		\qquad
		\gamma \asymp n^{-\frac{1}{(2\alpha+d)(1+q)}}, 
		\qquad
		\sigma \asymp 4 M_0 \vee c_1,
	\end{align*}
	we obtain
	\begin{align*}
		\|\wideparen{f}_D - f^*\|^2_{L_2(P_X)}  
		\lesssim n^{-\frac{2\alpha}{(2\alpha+d)(1+q)}}
	\end{align*}
	with probability at least $1 - 1/n$. This finishes the proof.  
\end{proof}

\subsection{Proofs Related to Section \ref{sec::ErrorAnalysis}}\label{subsec::prooferror}

\subsubsection{Proofs Related to Section \ref{sec::sampleerror}}

Before proceeding, we need to introduce the concept of entropy numbers \citep{vaart1997weak}, which serves as a measure of the capacity of a function set.

\begin{definition}[Entropy Numbers] \label{def::entropy numbers}
	Let $(\mathcal{X}, d)$ be a metric space, and let $A \subset \mathcal{X}$ with $i \geq 1$ being an integer. The $i$-th entropy number of $(A, d)$ is defined as
	\begin{align*}
		e_i(A, d) = \inf \left\{ \varepsilon > 0 : \exists x_1, \ldots, x_{2^{i-1}} \in \mathcal{X} \text{ such that } A \subset \bigcup_{j=1}^{2^{i-1}} B_d(x_j, \varepsilon) \right\},
	\end{align*}
	where $B_d(x_j, \varepsilon) = \{ x \in \mathcal{X} : d(x, x_j) \leq \varepsilon \}$ denotes the closed ball of radius $\varepsilon$ centered at $x_j$.
\end{definition}

The following lemma, derived from \cite[Theorem 6.26]{steinwart2008support}, provides an upper bound for the entropy number of Gaussian kernels.

\begin{lemma}\label{lem::entropygaussian}
	Let the compact set $\mathcal{X} \subset \mathbb{R}^d$ and let $P_X$ be a distribution defined on $\mathcal{X}$, with $\mathrm{supp}(P_X) \subset \mathcal{X}$ representing the support of $P_X$. Additionally, for $\gamma \in (0,1)$, let $H_{\gamma}(A)$ denote the reproducing kernel Hilbert space (RKHS) associated with the Gaussian radial basis function (RBF) kernel $k_{\gamma}$ over the set $A$. Then, for every $N \in \mathbb{N}^*$, there exists a constant $c_{N,d} > 0$ such that
	\begin{align*}
		e_i(\mathrm{id}: H_{\gamma}(\mathcal{X}) \to L_2(P_X)) \leq c_{N,d} \gamma^{-N} i^{-\frac{N}{d}}, 
		\qquad 
		\text{for } i > 1.
	\end{align*}
\end{lemma}

\begin{proof}[of Lemma \ref{lem::entropygaussian}]
	Consider the following commutative diagram:
	\begin{align*}
		\xymatrix{
			H_{\gamma}(\mathcal{X}) \ar[rr]^{\mathrm{id}} \ar[d]_{\mathcal{I}_{\mathrm{supp}(P_X)}} & & L_2(P_X) 
			\\
			H_{\gamma}(\mathrm{supp}(P_X)) \ar[rr]_{\mathrm{id}} & & \ell_{\infty}(\mathrm{supp}(P_X)) \ar[u]_{\mathrm{id}}
		}
	\end{align*}
	In this diagram, the extension operator $\mathcal{I}_{\mathrm{supp}(P_X)}: H_{\gamma}(\mathcal{X}) \to H_{\gamma}(\mathrm{supp}(P_X))$, as defined in Corollary 4.43 of \cite{steinwart2008support}, is an isometric isomorphism. This implies that 
	\begin{align} \label{eq::IPX}
		\|\mathcal{I}_{\mathrm{supp}(P_X)}: H_{\gamma}(\mathcal{X}) \to H_{\gamma}(\mathrm{supp}(P_X))\| = 1.
	\end{align}

	Let $\ell_{\infty}(\mathcal{X})$ denote the space of all bounded functions on the set $\mathcal{X}$. For any $f \in \ell_{\infty}(\mathcal{X})$, we have 
	\begin{align*}
		\|f\|_{L_2(P_X)} = \biggl( \mathbb{E}_{P_X} |f(X)|^2 \biggr)^{1/2} \leq \|f\|_{\infty}.
	\end{align*}
	This implies
	\begin{align} \label{eq::idinf2}
		\|\mathrm{id} : \ell_{\infty}(\text{supp}(\mathcal{X})) \to L_2(P_X)\| \leq 1.
	\end{align}
	
	Combining \eqref{eq::IPX}, \eqref{eq::idinf2} with Inequalities (A.38), (A.39), and Theorem 6.27 from \cite{steinwart2008support}, we obtain the following bound for all $i \geq 1$ and $N \geq 1$:
	\begin{align*}
		e_i(\mathrm{id} : H_{\gamma}(\mathcal{X}) \to L_2(P_X))
		& \leq \|\mathcal{I}_{\mathrm{supp}(P_X)} : H_{\gamma}(\mathcal{X}) \to H_{\gamma}(\text{supp}(\mathcal{X}))\|
		\\
		& \phantom{=}
		\cdot e_i(\mathrm{id} :  H(\mathrm{supp}(P_X)) \to \ell_{\infty}(\mathrm{supp}(P_X)) )
		\\
		& \phantom{=}
		\cdot \|\mathrm{id} : \ell_{\infty}(\text{supp}(\mathcal{X})) \to L_2(P_X)\|
		\\
		& \leq c_{N,d} \gamma^{-N}i^{-\frac{N}{d}},
	\end{align*}
	where $c_{N,d}$ is the constant specified in \cite[Theorem 6.27]{steinwart2008support}.
\end{proof}

Before proceeding, we need to introduce some notations. Specifically, we define the composition $(L \circ f)(x,y) := L(y,f(x))$ and introduce the function 
\begin{align}\label{eq::hf}
	h_f := L \circ f - L \circ f^*.
\end{align}
Let $H$ denote the reproducing kernel Hilbert space (RKHS), and let $r > 0$. We define the function space 
\begin{align}\label{eq::Fr}
	\mathcal{F}_r 
	:= \bigl\{ f \in H : \lambda \|f\|_H^2 + \mathcal{R}_{L,P}(\wideparen{f}) - \mathcal{R}_{L,P}^* \leq r \bigr\}
\end{align}
and 
\begin{align}\label{eq::Hr}
	\mathcal{H}_r 
	:= \bigl\{ h_{\wideparen{f}} := L \circ \wideparen{f} - L \circ f^* : f \in \mathcal{F}_r \bigr\}.
\end{align}

Additionally, we need to introduce a concept for measuring the capacity of a function set. This is defined as an expectation of the supremum with respect to the Rademacher sequence, see e.g., Definition 7.9 of \cite{steinwart2008support}.

\begin{definition}[Empirical Rademacher Average] \label{def::RademacherDefinition}
	Let $\mathcal{F}$ be a set of functions $f:\mathcal{Z}\to \mathbb{R}$. Let $\{\varepsilon_i\}_{i=1}^m$ be a Rademacher sequence associated with a distribution $\nu$. This sequence consists of independent and identically distributed (i.i.d.) random variables, where $\nu(\varepsilon_i = 1) = \nu(\varepsilon_i = -1) = \frac{1}{2}$. Then for a dataset $D:=(z_1, \ldots, z_n) \in \mathcal{Z}^n$, the $n$-th empirical Rademacher average of the function set $\mathcal{F}$ with respect to $D$ is defined as 
	\begin{align*}
		\mathrm{Rad}_D (\mathcal{F}, n)
		:= \mathbb{E}_{\nu} \sup_{f \in \mathcal{F}} 
		\biggl| \frac{1}{n} \sum_{i=1}^n \varepsilon_i f(z_i) \biggr|.
	\end{align*}
\end{definition}

To derive the bound of the empirical Rademacher average of $\mathcal{H}_r$, we first need to investigate the Lipschitz property of the Cauchy loss function. This involves establishing a supremum bound on the difference between the Cauchy loss values associated with two different regressors.

\begin{proof}[of Lemma \ref{lem::lip}]
	Let us define the function 
	\begin{align*}
		h(t) := \sigma^2 \log\left(1 + \frac{(y - t)^2}{\sigma^2}\right),
		\qquad
		\text{ for }
		t \in \mathbb{R},
		\,
		\sigma > 0.
	\end{align*}
	Taking the derivative of $h(t)$ with respect to $t$, we get
	\begin{align*}
		h'(t) = \frac{2 \sigma^2 (t - y)}{\sigma^2 + (t - y)^2}
		\in [-\sigma, \sigma]
		\qquad
		\text{ for any } t \in \mathbb{R},
	\end{align*}
	since it holds that $\sigma^2 + (t - y)^2 \geq 2\sigma |t - y|$. By applying the Mean Value Theorem, we can find some value $\xi$ between $f(x)$ and $g(x)$ such that
	\begin{align*}
		h(f(x)) - h(g(x)) 
		= h'(\xi) (f(x) - g(x)).
	\end{align*}
	This leads to
	\begin{align*}
		|h(f(x)) - h(g(x))| 
		= |h'(\xi)| \cdot |f(x) - g(x)| 
		\leq \sigma \cdot |f(x) - g(x)|,
	\end{align*}
	yielding the desired assertion.
\end{proof}

\begin{proof}[of Lemma \ref{lem::variancebound}]
	Using the Lipschitz property of the Cauchy loss established in Lemma \ref{lem::lip} and the refined calibration inequality from Theorem \ref{lem::rell2Cauchylargesigma}, we obtain 
	\begin{align*}
		\mathbb{E} |L(y, \wideparen{f}(x)) - L(y, f^*(x))|^2 
		\leq \sigma^2 \cdot \|\wideparen{f} - f^*\|^2_{L_2(P_X)} 
		\leq 8 \sigma^2 \cdot \bigl(\mathcal{R}_{L,P}(\wideparen{f}) - \mathcal{R}_{L,P}^*\bigr).
	\end{align*}
	This completes the proof.
\end{proof}

\begin{lemma}\label{lem::RadHr}
	Let the function space $\mathcal{H}_r$ be defined as in \eqref{eq::Hr}. For any $q \in (0, 1)$ and $\delta \in (0, 1/2)$, there exists a constant $c_1 > 0$ that depends only on $\delta$, $d$, and $q$ such that for any $\sigma \geq 4M \vee c_1$, we have
	\begin{align*}
		\mathbb{E}_{D \sim P^n} \mathrm{Rad}_D (\mathcal{H}_r, n) \leq \psi_n(r),
	\end{align*}
	where 
	\begin{align*}
		\psi_n(r) := c_2 \Bigl( \lambda^{-\frac{q}{2}} \sigma \gamma^{-\frac{d}{2}} n^{-\frac{1}{2}} r^{\frac{1}{2}} \vee \left( \frac{r}{\lambda} \right)^{\frac{q}{1+q}} \sigma M^{\frac{1-q}{1+q}} \gamma^{-\frac{d}{1+q}} n^{-\frac{1}{1+q}} \Bigr),
	\end{align*}
	and $c_2$ is a constant that depends on $d$ and $q$.
\end{lemma}

\begin{proof}[of Lemma \ref{lem::RadHr}]
	By applying Lemma \ref{lem::entropygaussian} and setting $q := \frac{d}{N}$, we obtain 
	\begin{align}\label{eq::ei}
		e_i(\mathrm{id}: H(\mathcal{X})\to \ell_{\infty}(\mathcal{X})) 
		\leq c_{d/q,d} \gamma^{-d/q} i^{-1/q}, 
		\qquad 
		\forall \, i > 1, \, q > 0.
	\end{align}
	To avoid confusion, we denote the constant $c_{d/q,d}$ that depends on $d$ and $q$ simply as $c_{q,d}$ for convenience. For any $f \in \mathcal{F}_r$, we have
	\begin{align*}
		\lambda \|f\|_H^2 \leq \lambda \|f\|_H^2 + \mathcal{R}_{L,P}(\wideparen{f}) - \mathcal{R}_{L,P}^* \leq r.
	\end{align*}
	From this, we deduce that 
	\begin{align*}
		\|f\|_H \leq \left( \frac{r}{\lambda} \right)^{1/2}, 
	\end{align*}
	which yields that
	\begin{align*}
		\mathcal{F}_r \subset \left( \frac{r}{\lambda} \right)^{1/2} B_H,
	\end{align*}
	where the unit ball $B_H := \{f\in H: \|f\|_H \leq 1\}$.
	By applying (A.36) from \cite{steinwart2008support} along with \eqref{eq::ei}, we obtain
	\begin{align*}
		e_i(\mathcal{F}_r, L_2(D_X)) 
		\leq 2\left(\frac{r}{\lambda}\right)^{1/2} e_i(\mathrm{id}: H(\mathcal{X}) \to \ell_{\infty}(\mathcal{X}))
		\leq 2\left(\frac{r}{\lambda}\right)^{1/2} c_{q,d} \gamma^{-d/q} i^{-1/q}.
	\end{align*}
	
	Let $\{ f_1, \ldots, f_{2^i} \}$ be an $\varepsilon$-net of $\mathcal{F}_r$ with respect to $L_2(D_X)$. For any $h_f \in \mathcal{F}_r$, there exists some index $j \in \{1, \ldots, 2^i\}$ such that 
	\begin{align*}
		\|f - f_j\|_{L_2(D_X)} \leq \varepsilon.
	\end{align*}
	By applying Lemma \ref{lem::lip}, we get
	\begin{align*}
		|h_f(x,y) - h_{f_j}(x,y)| 
		= |L(y,f(x)) - L(y,f_j(x))| 
		\leq \sigma |f(x) - f_j(x)|
	\end{align*}
	Consequently, we obtain
	\begin{align*}
		\|h_f - h_{f_j}\|_{L_2(D)} 
		\leq \sigma \|f - f_j\|_{L_2(D)} 
		\leq \sigma \varepsilon. 
	\end{align*}
	As a result, the set $\{h_{f_1}, \ldots, h_{f_{2^i}}\}$ constitutes a $\sigma \varepsilon$-net of $\mathcal{H}_r$ with respect to $L_2(D_X)$. This leads us to 
	\begin{align*}
		e_i(\mathcal{H}_r, L_2(D)) 
		\leq \sigma \cdot e_i(\mathcal{F}_r, L_2(D)) 
		\leq 2\sigma \cdot c_{q,d} \left(\frac{r}{\lambda}\right)^{1/2} \gamma^{-d/q} i^{-1/q}.
	\end{align*}
	Additionally, for any function $f$, we have 
	\begin{align*}
		\|\wideparen{f} - f^*\|_{\infty} \leq M + \|f^*\|_{\infty} \leq 2M.
	\end{align*}
	Applying Lemma \ref{lem::lip}, we get
	\begin{align}\label{eq::hfinftynorm}
		\|h_{\wideparen{f}}\|_{\infty} \leq \sigma \|\wideparen{f}-f^*\|_{\infty} \leq 2\sigma M.
	\end{align}
	Furthermore, by Lemma \ref{lem::variancebound}, if $\sigma \geq 4M \vee c_1$, then for any $f \in \mathcal{F}_r$, we have
	\begin{align*}
		\mathbb{E}_P h_{\wideparen{f}}^2 \leq 3\sigma^2 \mathbb{E}_P h_{\wideparen{f}} \leq 3\sigma^2 r.
	\end{align*}
	Using Theorem 7.16 in \cite{steinwart2008support}, we get 
	\begin{align*}
		& \mathbb{E}_{D \sim P^n} \mathrm{Rad}_D (\mathcal{H}_r, n) 
		\\
		& \leq c_q \Bigl( \Bigl( 2 \sigma c_{q,d} \left( \frac{r}{\lambda} \right)^{\frac{1}{2}} \gamma^{-\frac{d}{2q}} \Bigr)^q \left( 3 \sigma^2 r \right)^{(1-q)/2} n^{-\frac{1}{2}}
		\vee \Bigl( 2 \sigma c_{q,d} \left( \frac{r}{\lambda} \right)^{\frac{1}{2}} \gamma^{-\frac{d}{2q}} \Bigr)^{\frac{2q}{1+q}} \left( 2 \sigma M \right)^{\frac{1-q}{1+q}} n^{-\frac{1}{1+q}} \Bigr)
		\nonumber\\
		& \leq c_2 \Bigl( \lambda^{-\frac{q}{2}} \sigma \gamma^{-\frac{d}{2}} n^{-\frac{1}{2}} r^{\frac{1}{2}}
		\vee \left( \frac{r}{\lambda} \right)^{\frac{q}{1+q}}  \sigma M^{\frac{1-q}{1+q}} \gamma^{-\frac{d}{1+q}} n^{-\frac{1}{1+q}}\Bigr),
	\end{align*}
	where $c_2 := 8c_q \bigl((2 c_{q,d})^q \vee (2 c_{q,d})^{2q/(1+q)}\bigr)$. This finishes the proof. 
\end{proof}

Before proving the oracle inequality in Proposition \ref{prop::oracle}, we need to introduce two widely-used concentration inequalities. Specifically, Bernstein's inequality is shown in \cite[Theorem 5.12]{steinwart2008support} and Talagrand's inequality is proven by Theorem 7.5 and Lemma 7.6 in \cite{steinwart2008support}.

\begin{lemma}[Bernstein's Inequality]\label{lem::benstein}
	Let $\xi_1, \ldots, \xi_n$ be independent random variables on a probability space $(\mathcal{X}, \mathcal{A}, P)$ such that $\mathbb{E}_P \xi_i = 0$, $\|\xi_i\|_{\infty} \leq B$, and $\mathbb{E}_P \xi_i^2 = \sigma^2$ for all $i \in [n]$. Then for any $\tau > 0$, we have 
	\begin{align*}
		\frac{1}{n}\sum_{i=1}^n \xi_i \leq \sqrt{\frac{2\sigma^2\tau}{n}} + \frac{2B\tau}{3n}
	\end{align*}
	with probability $P^n$ at least $1 - e^{-\tau}$.
\end{lemma}

Let us define 
\begin{align}\label{eq::gfr}
	g_{f,r} 
	:= \frac{\mathbb{E}_P h_{\wideparen{f}} - h_{\wideparen{f}}}{\lambda \|f\|_H^2 + \mathbb{E}_P h_{\wideparen{f}} + r}, 
	\qquad 
	f \in H, \, r > 0. 
\end{align}

\begin{lemma}[Talagrand's Inequality]\label{lem::talagrand}
	For a given $r>0$, let $g_{f,r}$ be as in \eqref{eq::gfr} and define $\mathcal{G}:=\{ g_{f,r}: f\in \mathcal{H}\}$. For any $g \in \mathcal{G}$ such that $\mathbb{E}_P g = 0$, we assume $\|g\|_{\infty} \leq B$ and $\mathbb{E}_P g^2 \leq \sigma^2$. For $n \geq 1$, we define the function $G: \mathcal{Z}^n \to \mathbb{R}$ by
	\begin{align*}
		G(z_1, \ldots, z_n) := \sup_{g \in \mathcal{G}} \biggl| \frac{1}{n} \sum_{i=1}^n g(z_i) \biggr|, 
		\qquad 
		z = (z_1, \ldots, z_n) \in \mathcal{Z}^n.
	\end{align*}
	Then, for any $\tau > 0$, we have 
	\begin{align*}
		G(z) \leq \frac{5}{4} \cdot \mathbb{E}_{P^n} G(z) + \sqrt{\frac{2 \sigma^2\tau}{n}} + \frac{14B\tau}{3n}
	\end{align*}
	with probability $P^n$ at least $1 - e^{-\tau}$.
\end{lemma}

\begin{proof}[of Proposition \ref{prop::oracle}]
	According to the definitions of $f_D$ in \eqref{eq::KCRR} and $h_f$ in \eqref{eq::hf}, we have
	\begin{align}\label{eq::optimality}
		\lambda \|f_D\|_H^2 + \mathbb{E}_D h_{\wideparen{f}_D} 
		\leq \lambda \|f_0\|_H^2 + \mathbb{E}_D h_{\wideparen{f}_0}.
	\end{align}
	Consequently we obtain
	\begin{align}\label{eq::oradecomp}
		\lambda \|f_D\|_H^2 & + \mathcal{R}_{L,P}(\wideparen{f}_D) - \mathcal{R}_{L, P}^* 
		= \lambda \|f_D\|_H^2 + \mathbb{E}_P h_{\wideparen{f}_D} 
		\nonumber\\
		& = \lambda \|f_D\|_H^2 + \mathbb{E}_D h_{\wideparen{f}_D} - \mathbb{E}_D h_{\wideparen{f}_D} + \mathbb{E}_P h_{\wideparen{f}_D}
		\nonumber\\
		& \leq \lambda \|f_0\|_H^2 + \mathbb{E}_D h_{\wideparen{f}_0} - \mathbb{E}_D h_{\wideparen{f}_D} + \mathbb{E}_P h_{\wideparen{f}_D}
		\nonumber\\
		& = \lambda \|f_0\|_H^2 + \mathbb{E}_P h_{\wideparen{f}_0} + \left( \mathbb{E}_D h_{\wideparen{f}_0} - \mathbb{E}_P h_{\wideparen{f}_0} \right) + \left( \mathbb{E}_P h_{\wideparen{f}_D} - \mathbb{E}_D h_{\wideparen{f}_D} \right).
	\end{align}

	We begin by bounding the term $\mathbb{E}_{P} h_{\wideparen{f}_D} - \mathbb{E}_{P} h_{\wideparen{f}_0}$. For $i \in [n]$, define random variables
	\begin{align*}
		\xi_i := h_{\wideparen{f}_0}(X_i, Y_i) - \mathbb{E}_P h_{\wideparen{f}_0}(X, Y). 
	\end{align*}
	It is clear that $\mathbb{E}_P \xi_i = 0$.
	By applying \eqref{eq::hfinftynorm} and Lemma \ref{lem::variancebound}, we find that if $\sigma \geq 4M \vee c_1$, then $\|\xi_i\|_{\infty} \leq 2\|h_{\wideparen{f}}\|_{\infty} \leq 4\sigma M$. Additionally, we have $\mathbb{E}_P \xi_i^2 \leq \mathbb{E}_P h_{\wideparen{f}_0}^2 \leq 3\sigma^2 \mathbb{E}_P h_{\wideparen{f}_0}$.
	By applying Bernstein's inequality from Lemma \ref{lem::benstein} to random variables $(\xi_i)_{i=1}^n$, and utilizing the inequality $2ab \leq (a + b)^2$, we obtain 
	\begin{align}\label{eq::EDEPfP}
		\mathbb{E}_D h_{\wideparen{f}_0} - \mathbb{E}_{P} h_{\wideparen{f}_0} 
		\leq \sqrt{\frac{6\sigma^2 \mathbb{E}_P h_{\wideparen{f}_0} \tau}{n}} + \frac{8\sigma M \tau}{3n}
		\leq \mathbb{E}_P h_{\wideparen{f}_0} + \frac{3\sigma^2 \tau}{n} + \frac{8\sigma M \tau}{3n}
	\end{align}
	with probability at least $1 - e^{-\tau}$.

	Then we bound the term $\mathbb{E}_{P} h_{\wideparen{f}_D} - \mathbb{E}_D h_{\wideparen{f}_D}$. To this end, let us define 
	\begin{align*}
		g_{f,r} 
		:= \frac{\mathbb{E}_P h_{\wideparen{f}} - h_{\wideparen{f}}}{\lambda \|f\|_H^2 + \mathbb{E}_P h_{\wideparen{f}} + r}, 
		\qquad 
		f \in H, \, r > r^*, 
	\end{align*}
	where $r^* := \inf\{\mathbb{E}_P h_{\wideparen{f}}: f\in H\}$.
	Symmetrization in Proposition 7.10 of \cite{steinwart2008support} and Lemma \ref{lem::RadHr} yield 
	\begin{align*}
		\mathbb{E}_{D\sim P^n} \sup_{f\in\mathcal{F}_r} \left| \mathbb{E}_D \left( \mathbb{E}_P h_{\wideparen{f}} - h_{\wideparen{f}} \right) \right| 
		\leq 2 \mathbb{E}_{D \sim P^{n}} \mathrm{Rad}_D(\mathcal{H}_r, n) \leq 2\psi_{n}(r),
	\end{align*}
	where $\psi_n(r)$ is defined as in Lemma \ref{lem::RadHr}.
	Simple calculation shows that $\psi_{n}(4r) \leq 2\psi_{n}(r)$. Additionally, note that
	$\mathcal{H}_r$ is a separable Caratheodory set according to  Lemma 7.6 in \cite{steinwart2008support}.
	Applying Peeling in Theorem 7.7 of \cite{steinwart2008support} to $\mathcal{F}_r$ hence gives
	\begin{align*}
		\mathbb{E}_{D \sim P^n} \sup_{f \in H} \left| \mathbb{E}_D g_{f,r} \right| 
		\leq \frac{8 \psi_n(r)}{r}.
	\end{align*}
	By \eqref{eq::hfinftynorm}, we have 
	\begin{align*}
		\|g_{f,r}\|_{\infty} 
		& \leq \frac{2 \|h_{\wideparen{f}}\|_{\infty}}{r}
		\leq \frac{4 \sigma M}{r}.
	\end{align*}
	Using $(a+b)^2 \geq 4ab$ and Lemma \ref{lem::variancebound}, we get 
	\begin{align*}
		\mathbb{E}_P g_{f,r}^2 
		\leq \frac{\mathbb{E}_P (h_{\wideparen{f}})^2}{(\mathbb{E}_P h_{\wideparen{f}} + r)^2} 
		\leq \frac{3\sigma^2 \mathbb{E}_P h_{\wideparen{f}}}{4r \mathbb{E}_P h_{\wideparen{f}}} 
		= \frac{3\sigma^2}{4r}.
	\end{align*}
	By applying Talagrand’s inequality as stated in Lemma \ref{lem::talagrand}, we can conclude that for any $r > r^*$, 
	\begin{align*}
		\sup_{f \in H} \mathbb{E}_D g_{f,r} 
		\leq \frac{10 \psi_n(r)}{r} + \sqrt{\frac{3\sigma^2 \tau}{2n r}} + \frac{56\sigma M \tau}{3n r}
	\end{align*}
	holds with probability at least $1 - e^{-\tau}$.
	Based on the definition of $g_{f_D, r}$, we have
	\begin{align} \label{eq::SecondTerm}
		\mathbb{E}_P h_{\wideparen{f}_D} - \mathbb{E}_D h_{\wideparen{f}_D} 
		& \leq \left( \lambda \|f_D\|_H^2 + \mathbb{E}_P h_{\wideparen{f}_D} \right) \left( \frac{10 \psi_n(r)}{r} + \sqrt{\frac{3\sigma^2 \tau}{2n r}} + \frac{56\sigma M \tau}{3n r} \right)
		\nonumber\\
		& \phantom{=}
		+ 10 \psi_n(r) + \sqrt{\frac{3\sigma^2 \tau r}{2n}} + \frac{56\sigma M \tau}{3n}
	\end{align}
	with probability at least $1 - e^{-\tau}$.

	Combining \eqref{eq::SecondTerm} with \eqref{eq::oradecomp} and \eqref{eq::EDEPfP}, we obtain
	\begin{align} \label{eq::SecondRound}
		\lambda \|f_D\|_H^2 + \mathbb{E}_P h_{\wideparen{f}_D}
		& \leq \lambda \|f_0\|_H^2 + 2\mathbb{E}_P h_{\wideparen{f}_0} + \frac{3\sigma^2 \tau}{n} + \frac{8\sigma M \tau}{3n}
		\nonumber\\
		& \phantom{=}
		+ \left( \lambda \|f_D\|_H^2 + \mathbb{E}_P h_{\wideparen{f}_D} \right) \left( \frac{10 \psi_n(r)}{r} + \sqrt{\frac{3\sigma^2 \tau}{2n r}} + \frac{56\sigma M \tau}{3n r} \right)
		\nonumber\\
		& \phantom{=}
		+ 10 \psi_n(r) + \sqrt{\frac{3\sigma^2 \tau r}{2n}} + \frac{56\sigma M \tau}{3n}
	\end{align}
	with probability at least $1 - 2e^{-\tau}$.
	To bound the terms in \eqref{eq::SecondRound}, we note that since $\sigma > M$, if $r \geq (30 c_2)^2 \sigma^2 \lambda^{-q} \gamma^{-d} n^{-1}$, then it follows that $r \geq 30 \psi_n(r)$. This implies that 
	\begin{align*}
		\frac{10 \psi_n(r)}{r} \leq \frac{1}{3}.
	\end{align*}
	Furthermore, if we set $r \geq (108 \sigma^2 \tau + 152 \sigma M \tau) n^{-1}$, we can derive the following inequalities:
	\begin{align*}
		\frac{3 \sigma^2 \tau}{n} \leq \frac{r}{3},
		\qquad
		\sqrt{\frac{3 \sigma^2 \tau}{nr}} \leq \frac{1}{6},
		\qquad
		\frac{56 \sigma M \tau}{3 n} \leq \frac{r}{8}.
	\end{align*}
	These estimates allow us to conclude that we obtain
	\begin{align}\label{eq::intereq}
		\lambda \|f_D\|_H^2 + \mathbb{E}_P h_{\wideparen{f}_D}
		& \leq \lambda \|f_0\|_H^2 + 2\mathbb{E}_P h_{\wideparen{f}_0} + \frac{2r}{3} 
		\nonumber\\
		& \phantom{=} 
		+ \left( \lambda \|f_D\|_H^2 + \mathbb{E}_P h_{\wideparen{f}_D} \right) \left( \frac{1}{3} + \frac{1}{6} + \frac{1}{8} \right) 
		+ \left( \frac{1}{3} + \frac{1}{6} + \frac{1}{8} \right) r
	\end{align}
	with probability at least $1 - 2 e^{-\tau}$.
	Elementary calculation yields that
	\begin{align*}
		\lambda \|f_D\|_H^2 + \mathbb{E}_{P} h_{\wideparen{f}_D}
		& \leq 6 \left( \lambda \|f_0\|_H^2 + \mathbb{E}_P h_{\wideparen{f}_0} \right) + 4r 
	\end{align*}
	holds with probability at least $1 - 2 e^{- \tau}$. 
	Let us define 
	\begin{align*}
		r := 152 c_2^2 \left( \sigma^2 \tau + \sigma M \tau + \sigma^2 \lambda^{-q} \gamma^{-d} \right) n^{-1}.
	\end{align*}
	With this definition, we obtain
	\begin{align*}
		\lambda \|f_D\|_H^2 + \mathbb{E}_P h_{\wideparen{f}_D} 
		& \leq 6 \left( \lambda \|f_0\|_H^2 + \mathbb{E}_P h_{\wideparen{f}_0} \right) + 
		608 c_2^2 \left( \sigma^2 \tau + \sigma M \tau + \sigma^2 \lambda^{-q} \gamma^{-d} \right) n^{-1} 
		\nonumber\\
		&\leq 6 \left( \lambda \|f_0\|_H^2 + \mathbb{E}_P h_{\wideparen{f}_0} \right) + 
		1216 c_2^2 \left( \sigma^2 \tau + \sigma^2 \lambda^{-q} \gamma^{-d} \right) n^{-1} 
	\end{align*}
	with probability at least $1 - 2 e^{-\tau}$. The last inequality holds under the condition that $\sigma \geq M$. This concludes the proof.
\end{proof}

\subsubsection{Proofs Related to Section \ref{sec::approxerror}}

\begin{proof}[of Proposition \ref{prop::approx}]
	For a fixed parameter $\gamma > 0$, we define the function $K : \mathbb{R}^d \to \mathbb{R}$ by
	\begin{align}\label{Conv}
		K(x) :=
		\biggl( \frac{2}{ \gamma^2 \pi} \biggr)^{d/2} 
		\exp \biggl( - \frac{2\|x\|_2^2}{\gamma^2} \biggr).
	\end{align}
	For any $x \in \mathcal{X}$, the convolution of $K$ and $f^*$ at $x$ is
	\begin{align*}
		(K * f^*)(x) 
		& = \int_{\mathbb{R}^d}
		\biggl( \frac{2}{\gamma^2 \pi} \biggr)^{d/2} 
		\exp \biggl( - \frac{2 \|x - t\|^2}{\gamma^2} \biggr)  f^*(t) \, dt
		\\
		& = \int_{\mathbb{R}^d} \biggl( \frac{2}{\gamma^2 \pi} \biggr)^{d/2} 
		\exp \biggl( - \frac{2 \|h\|^2}{\gamma^2} \biggr) f^*(x+h) \, dh.
	\end{align*}
	Since the function $f^*$ has compact support and is bounded, it follows that $f^* \in L_2(\mathbb{R}^d)$. Combining this fact with Proposition 4.46 in \cite{steinwart2008support}, we obtain
	\begin{align}\label{eq::KconvhiinRKHS}
		K * f^* \in \mathcal{H}.
	\end{align} 
	Moreover, since 
	\begin{align*}
		\int_{\mathbb{R}^d} \biggl( \frac{2}{\gamma^2 \pi} \biggr)^{d/2} 
		\exp \biggl( - \frac{2 \|h\|^2}{\gamma^2} \biggr)\, dh = 1, 
	\end{align*}
	we have
	\begin{align*}
		f^*(x) = \int_{\mathbb{R}^d} \biggl( \frac{2}{\gamma^2 \pi} \biggr)^{d/2} 
		\exp \biggl( - \frac{2 \|h\|^2}{\gamma^2} \biggr) f^*(x) \, dh.
	\end{align*}
	Then for any $x\in \mathcal{X}$, we have
	\begin{align}\label{eq::diffKff}
		\bigl| K * f^*(x) - f^*(x) \bigr| 
		& = \biggl| 
		\int_{\mathbb{R}^d}  \biggl( \frac{2}{\gamma^2 \pi} \biggr)^{\frac{d}{2}} 
		\exp \biggl( - \frac{2\|h\|^2}{\gamma^2} \biggr) 
		\bigl( f^*(x+ h) - f^*(x) \bigr) \, dh \biggr|
		\nonumber\\
		& \leq
		\int_{\mathbb{R}^d}  \biggl( \frac{2}{\gamma^2 \pi} \biggr)^{\frac{d}{2}} 
		\exp \biggl( - \frac{2\|h\|^2}{\gamma^2} \biggr) 
		\bigl| f^*(x+ h) - f^*(x) \bigr| \, dh
		\nonumber\\
		& \leq
		\int_{\mathbb{R}^d}  \biggl( \frac{2}{\gamma^2 \pi} \biggr)^{\frac{d}{2}} 
		\exp \biggl( - \frac{2\|h\|^2}{\gamma^2} \biggr) 
		c_{\alpha} \|h\|^{\alpha} \, dh.
	\end{align}
	Using the rotation invariance of $x \mapsto \exp ( - 2 \|x\|_2^2 / \gamma^2 )$ and the property $\Gamma(1 + t) = t \Gamma(t)$ for $t > 0$, we obtain
	\begin{align}\label{eq::1/eta}
		\int_{\mathbb{R}^d}  \biggl( \frac{2}{\gamma^2 \pi} \biggr)^{\frac{d}{2}} 
		\exp \biggl( - \frac{2\|h\|^2}{\gamma^2} \biggr) c_{\alpha}\|h\|^{\alpha} \, dh 
		& = c_{\alpha} \Big(\frac{\gamma}{\sqrt{2}}\Big)^{\alpha} \int_{\mathbb{R}^d}  \biggl( \frac{1}{\pi} \biggr)^{\frac{d}{2}} 
		\exp \bigl( -\|h\|^2 \bigr) 
		\|h\|^{\alpha} \, dh
		\nonumber\\
		& = c_{\alpha} \Big(\frac{\gamma}{\sqrt{2}}\Big)^{\alpha} \frac{2}{\Gamma(d/2)} \int_{0}^{\infty} e^{-r^2} r^{\alpha+d-1} dr
		\nonumber\\
		& = \frac{c_{\alpha}}{\Gamma(d/2)} \Gamma\Big(\frac{d+\alpha}{2}\Big) 2^{-\alpha/2}\gamma^{\alpha}.
	\end{align}
	Let $f_0 := K * f^* \in H$. By Proposition 4.46 in \cite{steinwart2008support}, we have
	\begin{align}\label{eq::norm}
		\|f_0\|_H^2 
		= \pi^{-d/2} \gamma^{-d} \| f^* \|_{L_2}^2
		\leq c_4 \gamma^{-d},
	\end{align}
	where $c_4 := \pi^{-d/2} \mu^2(\mathcal{X}) \|f^*\|_{\infty}^2$. Combining \eqref{eq::diffKff} and \eqref{eq::1/eta}, we obtain 
	\begin{align*}
		|f_0(x) - f^*(x)| \leq c_5 \gamma^{\alpha}
	\end{align*}
	with $c_5 := c_{\alpha} \Gamma((d+\alpha)/2)/ \Gamma(d/2)$. This together with Lemma \ref{lem::relation} gives
	\begin{align*}
		\lambda \|f_0\|_H^2 + \mathcal{R}_{L,P}(f_0) - \mathcal{R}_{L,P}^* 
		\leq c_3 (\lambda \gamma^{-d} + \gamma^{2\alpha}),
	\end{align*}
	where the constant $c_3 := c_4 + c_5^2$. Furthermore, this conclusion, in conjunction with Lemma~\ref{lem::validateclip}, yields the desired assertion.
\end{proof}

\section{Conclusion}\label{sec::conclusion}

In this paper, we tackle the challenge of robust regression in situations where traditional noise assumptions, such as the existence of a finite absolute mean, are not applicable. We introduce a generalized Cauchy noise assumption that accommodates noise distributions with finite moments of any order, including heavy-tailed cases like Cauchy noise, which lacks a finite absolute mean. 
Through an analysis of the \textit{kernel Cauchy ridge regressor} (\textit{KCRR}), we establish a relationship between excess Cauchy risk and $L_2$-risk, demonstrating that these risks become equivalent when the scale parameter of the Cauchy loss is large. Building on this foundation, we show that the excess Cauchy risk bound of KCRR improves as the scale parameter decreases, particularly under the assumption of H\"{o}lder smoothness.
Moreover, we derive the almost minimax-optimal convergence rate for KCRR by effectively selecting a proper scale parameter of the Cauchy loss. This illustrates the robustness of the Cauchy loss in addressing various noise types. Our findings underscore the potential of KCRR as a reliable method for regression tasks in challenging noise environments.

\section*{Acknowledge} 
Annika Betken gratefully acknowledges financial support from the Dutch Research Council (NWO) through VENI grant 212.164.

\bibliography{arXiv}

\begin{thebibliography}{43}
\providecommand{\natexlab}[1]{#1}
\providecommand{\url}[1]{\texttt{#1}}
\expandafter\ifx\csname urlstyle\endcsname\relax
  \providecommand{\doi}[1]{doi: #1}\else
  \providecommand{\doi}{doi: \begingroup \urlstyle{rm}\Url}\fi

\bibitem[Aftab and Hartley(2015)]{aftab2015convergence}
Khurrum Aftab and Richard Hartley.
\newblock Convergence of iteratively re-weighted least squares to robust
  m-estimators.
\newblock In \emph{2015 IEEE Winter Conference on Applications of Computer
  Vision}, pages 480--487. IEEE, 2015.

\bibitem[Agrusa et~al.(2022)Agrusa, Kunkel, and Coleman]{agrusa2022robust}
Anjulie~S Agrusa, David~C Kunkel, and Todd~P Coleman.
\newblock Robust regression and optimal transport methods to predict
  gastrointestinal disease etiology from high resolution egg and symptom
  severity.
\newblock \emph{IEEE Transactions on Biomedical Engineering}, 69\penalty0
  (11):\penalty0 3313--3325, 2022.

\bibitem[Black and Anandan(1996)]{black1996robust}
Michael~J Black and Paul Anandan.
\newblock The robust estimation of multiple motions: Parametric and
  piecewise-smooth flow fields.
\newblock \emph{Computer Vision and Image Understanding}, 63\penalty0
  (1):\penalty0 75--104, 1996.

\bibitem[Brownlees et~al.(2015)Brownlees, Joly, and
  Lugosi]{Brownlees2015empirical}
Christian Brownlees, Emilien Joly, and G{\'a}bor Lugosi.
\newblock Empirical risk minimization for heavy-tailed losses.
\newblock \emph{The Annals of Statistics}, 43\penalty0 (6):\penalty0
  2507--2536, 2015.

\bibitem[Catoni(2012)]{catoni2012challenging}
Olivier Catoni.
\newblock Challenging the empirical mean and empirical variance: a deviation
  study.
\newblock In \emph{Annales de l'IHP Probabilit{\'e}s et statistiques},
  volume~48, pages 1148--1185, 2012.

\bibitem[Chen et~al.(2021)Chen, Jin, Li, and Xu]{chen2021generalized}
Peng Chen, Xinghu Jin, Xiang Li, and Lihu Xu.
\newblock A generalized {C}atoni’s {M}-estimator under finite $\alpha$-th
  moment assumption with $\alpha\in (1, 2)$.
\newblock \emph{Electronic Journal of Statistics}, 15\penalty0 (2):\penalty0
  5523--5544, 2021.

\bibitem[D'Orsi et~al.(2021)D'Orsi, Novikov, and Steurer]{d2021consistent}
Tommaso D'Orsi, Gleb Novikov, and David Steurer.
\newblock Consistent regression when oblivious outliers overwhelm.
\newblock In Marina Meila and Tong Zhang, editors, \emph{Proceedings of the
  38th International Conference on Machine Learning}, volume 139 of
  \emph{Proceedings of Machine Learning Research}, pages 2297--2306. PMLR,
  18--24 Jul 2021.

\bibitem[Eberts and Steinwart(2013)]{mona2013optimal}
Mona Eberts and Ingo Steinwart.
\newblock Optimal regression rates for {SVM}s using {G}aussian kernels.
\newblock \emph{Electronic Journal of Statistics}, 7:\penalty0 1--42, 2013.

\bibitem[Feng et~al.(2015)Feng, Huang, Shi, Yang, and
  Suykens]{feng2015learning}
Yunlong Feng, Xiaolin Huang, Lei Shi, Yuning Yang, and Johan~AK Suykens.
\newblock Learning with the maximum correntropy criterion induced losses for
  regression.
\newblock \emph{The Journal of Machine Learning Research}, 16\penalty0
  (30):\penalty0 993--1034, 2015.

\bibitem[Friedman(1991)]{friedman1991multivariate}
Jerome~H Friedman.
\newblock Multivariate adaptive regression splines.
\newblock \emph{The Annals of Statistics}, 19\penalty0 (1):\penalty0 1--67,
  1991.

\bibitem[Fu et~al.(2010)Fu, Zhou, and Wang]{fu2010research}
Tianhui Fu, Suihua Zhou, and Liqun Wang.
\newblock Research on {ELF} atmospheric noise suppression method.
\newblock In \emph{The Third International Conference on Computer Science and
  Information Technology}, volume~5, pages 542--546. IEEE, 2010.

\bibitem[G{\"u}lg{\"u}n and Larsson(2024)]{gulgun2023massive}
Ziya G{\"u}lg{\"u}n and Erik~G. Larsson.
\newblock Massive {MIMO} with {C}auchy noise: {C}hannel estimation, achievable
  rate and data decoding.
\newblock \emph{IEEE Transactions on Wireless Communications}, 23\penalty0
  (3):\penalty0 1929--1942, 2024.

\bibitem[Gy{\"o}rfi et~al.(2006)Gy{\"o}rfi, Kohler, Krzyzak, and
  Walk]{gyorfi2006distribution}
L{\'a}szl{\'o} Gy{\"o}rfi, Michael Kohler, Adam Krzyzak, and Harro Walk.
\newblock \emph{A Distribution-free Theory of Nonparametric Regression}.
\newblock Springer Science \& Business Media, 2006.

\bibitem[H{\"a}rdle and Marron(1985)]{hardle1985optimal}
Wolfgang H{\"a}rdle and James~Stephen Marron.
\newblock Optimal bandwidth selection in nonparametric regression function
  estimation.
\newblock \emph{The Annals of Statistics}, 13\penalty0 (4):\penalty0
  1465--1481, 1985.

\bibitem[Huber(1992)]{huber1992robust}
Peter~J Huber.
\newblock Robust estimation of a location parameter.
\newblock In \emph{Breakthroughs in Statistics: Methodology and Distribution},
  pages 492--518. Springer, 1992.

\bibitem[Idan and Speyer(2010)]{idan2010cauchy}
Moshe Idan and Jason~L Speyer.
\newblock Cauchy estimation for linear scalar systems.
\newblock \emph{IEEE Transactions on Automatic Control}, 55\penalty0
  (6):\penalty0 1329--1342, 2010.

\bibitem[Karaku{\c{s}} et~al.(2022)Karaku{\c{s}}, Kuruo{\u{g}}lu, Achim, and
  Alt{\i}nkaya]{karakucs2022cauchy}
Oktay Karaku{\c{s}}, Ercan~E Kuruo{\u{g}}lu, Alin Achim, and Mustafa~A
  Alt{\i}nkaya.
\newblock Cauchy-{R}ician model for backscattering in urban {SAR} images.
\newblock \emph{IEEE Geoscience and Remote Sensing Letters}, 19:\penalty0 1--5,
  2022.

\bibitem[Kelly et~al.(2007)Kelly, Longjohn, and Nottingham]{kelly2007uci}
Markelle Kelly, Rachel Longjohn, and Kolby Nottingham.
\newblock {UCI} machine learning repository, 2007.
\newblock \url{https://archive.ics.uci.edu}.

\bibitem[Khan et~al.(2021)Khan, Yaqoob, Zubair, Khan, Ahmad, and
  Alamri]{khan2021applications}
Dost~Muhammad Khan, Anum Yaqoob, Seema Zubair, Muhammad~Azam Khan, Zubair
  Ahmad, and Osama~Abdulaziz Alamri.
\newblock Applications of robust regression techniques: an econometric
  approach.
\newblock \emph{Mathematical Problems in Engineering}, 2021:\penalty0 1--9,
  2021.

\bibitem[Lam and Cheng(2021)]{lam2021robust}
Clifford Lam and Wenyu Cheng.
\newblock Robust mean and eigenvalues regularized covariance matrix estimation.
\newblock \emph{London School of Economics and Political Science}, 2021.

\bibitem[Laus et~al.(2018)Laus, Pierre, and Steidl]{laus2018nonlocal}
Friederike Laus, Fabien Pierre, and Gabriele Steidl.
\newblock Nonlocal myriad filters for cauchy noise removal.
\newblock \emph{Journal of Mathematical Imaging and Vision}, 60:\penalty0
  1324--1354, 2018.

\bibitem[Lee et~al.(2020)Lee, Yang, Lim, and Oh]{lee2020optimal}
Kyungjae Lee, Hongjun Yang, Sungbin Lim, and Songhwai Oh.
\newblock Optimal algorithms for stochastic multi-armed bandits with heavy
  tailed rewards.
\newblock In H.~Larochelle, M.~Ranzato, R.~Hadsell, M.F. Balcan, and H.~Lin,
  editors, \emph{Advances in Neural Information Processing Systems}, volume~33,
  pages 8452--8462. Curran Associates, Inc., 2020.

\bibitem[Ljajko et~al.(2023)Ljajko, Stojanovi{\'c}, To{\v{s}}i{\'c}, and
  Bo{\v{z}}ovi{\'c}]{ljajko2023cauchy}
Eugen Ljajko, Vladica Stojanovi{\'c}, Marina To{\v{s}}i{\'c}, and Ivan
  Bo{\v{z}}ovi{\'c}.
\newblock Cauchy split-break process: Asymptotic properties and application in
  securities market analysis.
\newblock \emph{UPB Scientific Bulletin, Series A: Applied Mathematics and
  Physics}, 85:\penalty0 139--154, 2023.

\bibitem[Mlotshwa et~al.(2022)Mlotshwa, van Deventer, and
  Bosman]{mlotshwa2022cauchy}
Thamsanqa Mlotshwa, Heinrich van Deventer, and Anna~Sergeevna Bosman.
\newblock Cauchy loss function: {R}obustness under {G}aussian and {C}auchy
  noise.
\newblock In \emph{Southern African Conference for Artificial Intelligence
  Research}, pages 123--138. Springer, 2022.

\bibitem[Pervez and Ali(2022)]{pervez2022robust}
Asif Pervez and Irfan Ali.
\newblock Robust regression analysis in analyzing financial performance of
  public sector banks: a case study of {I}ndia.
\newblock \emph{Annals of Data Science}, pages 1--15, 2022.

\bibitem[Pirtea et~al.(2021)Pirtea, Noja, Cristea, and
  Panait]{pirtea2021interplay}
Marilen~Gabriel Pirtea, Gra{\c{t}}iela~Georgiana Noja, Mirela Cristea, and
  Mirela Panait.
\newblock Interplay between environmental, social and governance coordinates
  and the financial performance of agricultural companies.
\newblock \emph{Agricultural Economics/Zem{\v{e}}d{\v{e}}lsk{\'a} Ekonomika},
  67\penalty0 (12), 2021.

\bibitem[Santamar{\'\i}a et~al.(2006)Santamar{\'\i}a, Pokharel, and
  Principe]{santamaria2006generalized}
Ignacio Santamar{\'\i}a, Puskal~P Pokharel, and Jos{\'e}~Carlos Principe.
\newblock Generalized correlation function: definition, properties, and
  application to blind equalization.
\newblock \emph{IEEE Transactions on Signal Processing}, 54\penalty0
  (6):\penalty0 2187--2197, 2006.

\bibitem[Shen et~al.(2021{\natexlab{a}})Shen, Jiao, Lin, Horowitz, and
  Huang]{shen2021deep}
Guohao Shen, Yuling Jiao, Yuanyuan Lin, Joel~L Horowitz, and Jian Huang.
\newblock Deep quantile regression: Mitigating the curse of dimensionality
  through composition.
\newblock \emph{arXiv preprint arXiv:2107.04907}, 2021{\natexlab{a}}.

\bibitem[Shen et~al.(2021{\natexlab{b}})Shen, Jiao, Lin, and
  Huang]{shen2021robust}
Guohao Shen, Yuling Jiao, Yuanyuan Lin, and Jian Huang.
\newblock Robust nonparametric regression with deep neural networks.
\newblock \emph{arXiv preprint arXiv:2107.10343}, 2021{\natexlab{b}}.

\bibitem[Shi et~al.(2020)Shi, Dong, and Guo]{shi2020cauchy}
Kehan Shi, Gang Dong, and Zhichang Guo.
\newblock Cauchy noise removal by nonlinear diffusion equations.
\newblock \emph{Computers \& Mathematics with Applications}, 80\penalty0
  (9):\penalty0 2090--2103, 2020.

\bibitem[Steinwart and Christmann(2008)]{steinwart2008support}
Ingo Steinwart and Andreas Christmann.
\newblock \emph{Support Vector Machines}.
\newblock Springer Science \& Business Media, 2008.

\bibitem[Steinwart et~al.(2006)Steinwart, Hush, and
  Scovel]{steinwart2006oracle}
Ingo Steinwart, Don Hush, and Clint Scovel.
\newblock An oracle inequality for clipped regularized risk minimizers.
\newblock In B.~Sch\"{o}lkopf, J.~Platt, and T.~Hoffman, editors,
  \emph{Advances in Neural Information Processing Systems}, volume~19, pages
  1321--1328. MIT Press, 2006.

\bibitem[Tipping(2001)]{tipping2001sparse}
Michael~E Tipping.
\newblock Sparse {B}ayesian learning and the relevance vector machine.
\newblock \emph{The Journal of Machine Learning Research}, 1\penalty0
  (Jun):\penalty0 211--244, 2001.

\bibitem[Tsakonas et~al.(2014)Tsakonas, Jaldén, Sidiropoulos, and
  Ottersten]{Convergence2014Tsakonas}
Efthymios Tsakonas, Joakim Jaldén, Nicholas~D. Sidiropoulos, and Björn
  Ottersten.
\newblock Convergence of the {H}uber regression {M}-estimate in the presence of
  dense outliers.
\newblock \emph{IEEE Signal Processing Letters}, 21\penalty0 (10):\penalty0
  1211--1214, 2014.

\bibitem[van~der Vaart and Wellner(1997)]{vaart1997weak}
AW~van~der Vaart and Jon~A Wellner.
\newblock Weak convergence and empirical processes with applications to
  statistics.
\newblock \emph{Journal of the Royal Statistical Society Series A (Statistics
  in Society)}, 160\penalty0 (3):\penalty0 596--608, 1997.

\bibitem[Van~Erven et~al.(2015)Van~Erven, Gr{\"u}nwald, Mehta, Reid, and
  Williamson]{van2015fast}
Tim Van~Erven, Peter~D Gr{\"u}nwald, Nishant~A Mehta, Mark~D Reid, and Robert~C
  Williamson.
\newblock Fast rates in statistical and online learning.
\newblock \emph{The Journal of Machine Learning Research}, 16\penalty0
  (1):\penalty0 1793--1861, 2015.

\bibitem[Wang et~al.(2014)Wang, Zhang, Chen, and Zhong]{wang2014least}
Kuaini Wang, Jingjing Zhang, Yanyan Chen, and Ping Zhong.
\newblock Least absolute deviation support vector regression.
\newblock \emph{Mathematical Problems in Engineering}, 2014\penalty0
  (1):\penalty0 169575, 2014.

\bibitem[Wang et~al.(2022)Wang, Wang, Peng, Li, and Yin]{wang2022huber}
Yue Wang, Baobin Wang, Chaoquan Peng, Xuefeng Li, and Hong Yin.
\newblock Huber regression analysis with a semi-supervised method.
\newblock \emph{Mathematics}, 10\penalty0 (20):\penalty0 3734, 2022.

\bibitem[Xu et~al.(2023)Xu, Yao, Yao, and Zhang]{xu2023non}
Lihu Xu, Fang Yao, Qiuran Yao, and Huiming Zhang.
\newblock Non-asymptotic guarantees for robust statistical learning under
  infinite variance assumption.
\newblock \emph{The Journal of Machine Learning Research}, 24\penalty0
  (92):\penalty0 1--46, 2023.

\bibitem[Xu et~al.(2020)Xu, Zhu, Yang, Zhang, Jin, and Yang]{xu2020learning}
Yi~Xu, Shenghuo Zhu, Sen Yang, Chi Zhang, Rong Jin, and Tianbao Yang.
\newblock Learning with non-convex truncated losses by {SGD}.
\newblock In Ryan~P. Adams and Vibhav Gogate, editors, \emph{Proceedings of The
  35th Uncertainty in Artificial Intelligence Conference}, volume 115 of
  \emph{Proceedings of Machine Learning Research}, pages 701--711. PMLR, 22--25
  Jul 2020.

\bibitem[Zhang et~al.(2021)Zhang, Liu, Wang, Fan, and
  Qiu]{zhang2021measurements}
Jiachi Zhang, Liu Liu, Kai Wang, Yuanyuan Fan, and Jiahui Qiu.
\newblock Measurements and statistical analyses of electromagnetic noise for
  industrial wireless communications.
\newblock \emph{International Journal of Intelligent Systems}, 36\penalty0
  (3):\penalty0 1304--1330, 2021.

\bibitem[Zhang and Zhou(2018)]{zhang2018ell_1}
Lijun Zhang and Zhi-Hua Zhou.
\newblock $\ell_1$-regression with heavy-tailed distributions.
\newblock In S.~Bengio, H.~Wallach, H.~Larochelle, K.~Grauman, N.~Cesa-Bianchi,
  and R.~Garnett, editors, \emph{Advances in Neural Information Processing
  Systems}, volume~31, pages 1084--1094. Curran Associates, Inc., 2018.

\bibitem[Zhao and Yang(2023)]{zhao2023minimax}
Bingxin Zhao and Yuhong Yang.
\newblock Minimax rates of convergence for nonparametric location-scale models.
\newblock \emph{arXiv preprint arXiv:2307.01399}, 2023.

\end{thebibliography}

\end{document}